\documentclass[lettersize,journal]{IEEEtran}
\usepackage{amsmath,amsfonts}
\usepackage{algorithmic}
\usepackage{algorithm}
\usepackage{array}
\usepackage[caption=false,font=normalsize,labelfont=sf,textfont=sf]{subfig}
\usepackage{textcomp}
\usepackage{stfloats}
\usepackage{url}
\usepackage{verbatim}
\usepackage{graphicx}
\usepackage{cite}
\usepackage{hyperref}
\usepackage{color}
\usepackage{xcolor}
\usepackage{multirow}
\usepackage{booktabs}
\usepackage{makecell}
\usepackage{graphicx}
\usepackage{amsthm,amsmath,amssymb}
\usepackage{mathrsfs}
\usepackage{pifont}
\newcommand{\name}{MEOW}
\newcommand{\nameplus}{AdaMEOW}

\newtheorem{definition}{Definition}
\newtheorem{theorem}{Theorem}

\newcommand{\yjx}[1]{{\color{black}#1}}

\begin{document}

\title{Heterogeneous Graph Contrastive Learning 
with Meta-path Contexts and 
Adaptively
Weighted Negative Samples}

\author{Jianxiang~Yu, Qingqing~Ge, Xiang~Li, Aoying Zhou
\thanks{Jianxiang Yu, Qingqing Ge, Xiang Li and Aoying Zhou are with the School
of Data Science and Engineering, East China Normal University, Shanghai. 
Xiang Li is the corresponding author. \\
E-mail: \{jianxiangyu,qingqingge\}@stu.ecnu.edu.cn,
\\
\{xiangli, ayzhou\}@dase.ecnu.edu.cn
}
}

\markboth{Journal of \LaTeX\ Class Files,~Vol.~14, No.~8, August~2021}%
{Shell \MakeLowercase{\textit{et al.}}: A Sample Article Using IEEEtran.cls for IEEE Journals}


\maketitle

\begin{abstract}
Heterogeneous graph contrastive learning has received wide attention recently. 
Some existing methods use meta-paths, which are sequences of object types that capture semantic relationships between objects, 
to construct contrastive views. 
However,
most of them ignore the rich meta-path context information that describes how two objects are connected by meta-paths. 
Further,
they fail to distinguish negative samples, 
which could adversely affect the model performance.
To address the problems,
we propose~\name, 
which considers both meta-path contexts and weighted negative samples. 
Specifically,
\name~constructs a coarse view and a fine-grained view for contrast. 
The former reflects which objects are connected by meta-paths,
while the latter uses meta-path contexts and characterizes details on how the objects are connected. 
Then,
we theoretically analyze the InfoNCE loss and recognize its limitations for computing gradients of negative samples. 
To better distinguish negative samples, 
we learn hard-valued weights for them based on node clustering and use prototypical contrastive learning to pull close embeddings of nodes in the same cluster. 
In addition, we propose a variant model~\nameplus~that adaptively learns soft-valued weights of negative samples to further improve node representation. 
Finally, we conduct extensive experiments to show the superiority of~\name~and~\nameplus~against other state-of-the-art methods.

\end{abstract}

\begin{IEEEkeywords}
Heterogeneous information network, contrastive learning, representation learning.
\end{IEEEkeywords}

\section{Introduction}
\IEEEPARstart{H}{eterogeneous} information networks (HINs) are 
prevalent in the real world,
such as social networks, citation networks, and knowledge graphs. 
In HINs, 
nodes (objects) are of different types to represent entities and edges (links) are also of multiple types to characterize various relations between entities.
For example,
in \emph{Facebook},
we have entities like \emph{users}, \emph{posts}, \emph{photos} and \emph{groups};
{users} can \emph{publish} {posts},
\emph{upload} photos and \emph{join} groups.
Compared with homogeneous graphs
where all the nodes and edges are of a single type,
HINs contain
richer semantics and more complicated structural information.
To further enrich the information of HINs,
nodes are usually associated with labels.
Since object labeling is costly,
graph neural networks (GNNs)
\cite{li2021leveraging,han,fu2020magnn} have recently been applied for classifying nodes in HINs and have shown to achieve superior performance.


Despite the success,
most existing heterogeneous graph neural network (HGNN) models 
require a large amount of training data, which is difficult to obtain.
To address the problem,
self-supervised learning,
which is in essence unsupervised learning,
has been applied in HINs 
\cite{park2022cross, hwang2020self}.
The core idea of self-supervised learning is to
extract supervision from data itself and learn high-quality representations with strong generalizability for downstream tasks.
In particular,
contrastive learning,
as one of the main self-supervised learning types,
has recently received significant attention.
Contrastive learning
aims to construct 
positive and negative pairs for contrast,
following the principle of
maximizing the mutual
information (MI)~\cite{oord2018representation} between positive pairs while minimizing that
between negative pairs.
Although some graph contrastive learning methods for HINs have already been proposed~\cite{park2022cross,jiang2021contrastive,hwang2020self},
most of them suffer from the following two main challenges: \emph{contrastive view construction} and \emph{negative sample selection}.


On the one hand, 
to construct contrastive views,
some methods
utilize 
meta-paths~\cite{zhao2021multi,heco}.
A meta-path,
which is a sequence of object
types, 
captures the semantic relation between objects in
HINs. 
For example, 
if we denote the object types \emph{User}
and \emph{Group} in Facebook as ``U'' and ``G'', respectively, 
the meta-path \emph{User-Group-User} (UGU) expresses the co-participation
relation. 
Specifically, 
two users $u_1$ and $u_2$ are UGU-related if a path instance $u_1-g-u_2$ exists, where $g$ is a group
object and describes 
the \emph{contextual} information on how $u_1$ and $u_2$ are connected.
The use of
meta-paths can identify a set of path-based neighbors
that are semantically related to a given object
and provide different views for contrast.
However,
existing contrastive learning methods omit the contextual information 
in each meta-path view.
For example,
HeCo~\cite{heco} 
takes meta-paths as views, but it 
only uses 
the fact that two objects are connected by meta-paths
and
discards the contexts of 
how 
they are semantically 
connected,
which we will call \emph{meta-path contexts}
and
can be very influential in the classification task. For example,
a group can provide valuable hints
on a user's topic interests.
Therefore,
contrasting meta-path views with rich contexts is a necessity.



On the other hand, 
negative sample selection is another 
challenge to be addressed.
Note that 
most existing graph contrastive learning methods~\cite{zeng2021contrastive,zhu2020deep,GraphCL} are formulated in a sampled noise contrastive estimation framework.
For each node in a view,
random negative sampling from the rest of intra-view and inter-view nodes is widely adopted.
However, 
this could introduce many \emph{easy negative samples} and
\emph{false negative samples}.
For easy negative samples,
they are less informative and easily lead to the vanishing gradient problem~\cite{zhang2019nscaching},
while
the
false negative samples
can adversely affect the learning process for providing incorrect information.
Recently, 
there exist some
works~\cite{qin2021relation,zhang2019nscaching,zhu2021structure} that seek to identify \emph{hard negative samples} for improving the discriminative power of encoders in HINs.
Despite their success,
most of them 
fail to distinguish hard negatives from false ones.
While ASA~\cite{qin2021relation}
is proposed to solve the issue,
it is specially designed for 
the link prediction task 
and can only generate negative samples for objects based on one type of relation in HINs,
which restricts its wide applicability.
Since there is not a clear cut boundary between false negatives and hard ones,
how to balance the exploitation of hard negative and false negative remains to be investigated.

In this paper,
to solve the two challenges,
we propose a heterogeneous graph contrastive learning method \name\  
with meta-path contexts and weighted negative samples.
Based on meta-paths,
we construct two novel views for contrast: 
\emph{the coarse view} and \emph{the fine-grained view}.
The coarse view expresses that two objects are connected by meta-paths,
while the fine-grained view utilizes meta-path contexts and  describes how they are connected.
In the coarse view,
we simply aggregate all the meta-paths and generate node embeddings that are taken as anchors.
In the fine-grained view,
we construct positive and negative samples for each anchor.
Specifically,
for each meta-path,
we first generate nodes' embeddings based on the meta-path induced graph.
To further improve the generalizability of the model,
we introduce noise by 
performing graph perturbations,
such as
\emph{edge masking} and 
\emph{feature masking},
on the meta-path induced graph to derive an augmented one,
based on which 
we also generate node embeddings.
In this way,
each meta-path generates two embedding vectors for each node.
After that,
for each node,
we fuse different embeddings from various meta-paths
to generate its final embedding vector.
Then for each anchor,
its embedding vector in the fine-grained view is taken as a positive sample while those of other nodes are considered as negative samples.
{
Subsequently,
based on theoretical analysis, 
we recognize that the InfoNCE loss lacks the ability to discriminate negative samples that have the same similarity with an anchor during training.
Therefore,
we perform node clustering 
and use the results to grade the weights of negative samples in~\name~to distinguish negative samples.
To further boost the model performance, 
we employ prototypical contrastive learning~\cite{li2020prototypical}, 
where the cluster centers, 
i.e., \emph{prototype vectors}, 
are used as positive/negative samples.
This helps learn 
compact embeddings for nodes
in the same cluster
by pushing nodes close to their
corresponding prototype vectors
and far away from other 
prototype vectors.
In addition,
since the weights are hard-valued in \name,
we further propose a variant model called \nameplus\ that can adaptively learn the soft-valued weights of negative samples, 
making negative samples more personalized and improving the learning ability of node representations.
}
Finally,
we summarize our contributions as:

\begin{itemize}
\item 
We propose a novel
heterogeneous graph contrastive learning model \name,
which constructs a coarse view and a fine-grained view for contrast based on meta-paths, respectively.
The former shows objects are connected by meta-paths, while the latter employs meta-path contexts and expresses how objects are connected by meta-paths.
{
\item We recognize the limitation of the InfoNCE loss based on theoretical analysis and propose a contrastive loss function with weighted negative samples to better distinguish 
negative samples.
\item 
We distinguish negative samples by
performing node clustering and using the results to grade their weights.
Based on the clustering results,
we also introduce prototypical contrastive learning to 
help learn compact embeddings
of nodes in the same cluster.
Further,
we propose a variant model,
namely, 
\nameplus,
which adaptively learns soft-valued weights for negative samples.
}
\item We conduct
extensive experiments comparing \name\ and \nameplus\ with other 9 state-of-the-art methods w.r.t. node classification and node  clustering tasks
on four public HIN datasets.
Our results show that \name\ 
achieves better performance than other competitors,
and \nameplus\ further improved on the basis of \name.
\end{itemize}

\section{Related Work}
\subsection{Heterogeneous Graph Neural Network}
Heterogeneous graph neural network (HGNN)
has recently received much attention
and there have been some models proposed.
For example,
HetGNN~\cite{zhang2019heterogeneous} aggregates information from neighbors of the same type with bi-directional LSTM to obtain type-level neighbor representations, 
and then fuses these neighbor representations with the attention mechanism.
HGT~\cite{hu2020heterogeneous} designs Transformer-like attention architecture
to calculate mutual attention of different neighbors.
HAN~\cite{han} employs both node-level and semantic-level attention mechanisms 
to learn the importance of neighbors under each meta-path 
and the importance of different meta-paths, respectively.
Considering meta-path contexts information,
MAGNN~\cite{fu2020magnn} improves HAN by employing a meta-path instance encoder to incorporate intermediate semantic nodes.
Further,
Graph Transformer Networks (GTNs) 
\cite{yun2019graph}
are capable of generating new graph structures, which can identify useful connections between unconnected nodes in the original graph and learn effective node representation in the new graphs.
Despite the success,
most of these methods are semi-supervised,
which heavily relies on labeled objects.

\subsection{Graph Contrastive Learning (GCL)}
Contrastive learning aims to construct positive and negative pairs for contrast,
whose goal is to pull close positive pairs while pushing away negative ones.
Recently, some works have applied contrastive learning to graphs~\cite{GraphSAGE,zhu2021graph}.
In particular,
most of these approaches use data augmentation 
to construct contrastive views 
and adopt the following three main contrast mechanisms:
(1) \emph{node-node contrast}~\cite{qiu2020gcc,jovanovic2021towards,tang2015line};
(2) \emph{graph-graph contrast}~\cite{suresh2021adversarial,zeng2021contrastive};
(3) \emph{node-graph contrast}~\cite{hassani2020contrastive,dmgi}.
For example,
GRACE~\cite{zhu2020deep} treats two augmented graphs by node feature masking and edge removing as two contrastive views
and then pulls the representation of the same nodes close while pushing the remaining nodes apart.
Inspired by SimCLR~\cite{chen2020simple} in the visual domain,
GraphCL~\cite{you2020graph} further extends this idea to 
graph-structured data, 
which relies on node dropping and edge perturbation 
to generate two perturbed graphs
and then maximizes the two graph-level mutual information (MI).
Moreover, 
DGI~\cite{dgi} is the first approach 
to propose the contrast between node-level embeddings and graph-level embeddings, 
which allows graph encoders to learn local and global semantic information.
In heterogeneous graphs,
HeCo~\cite{heco} takes two views from network schema and meta-paths to generate node representations and perform contrasts between nodes.
HDGI~\cite{ren2019heterogeneous} extends DGI to HINs and learns high-level node representations by maximizing MI between local and global representations.
However, most of these methods select negative samples by random sampling, which will introduce false negatives.
These samples will adversely affect the learning process, 
so we need to distinguish them from hard negatives.

\subsection{Hard Negative Sampling}
In contrastive learning, 
easy negative samples are easily distinguished from anchors, 
while hard negative ones are similar to anchors.
Recent studies~\cite{robinson2020contrastive} have shown that contrastive learning can benefit from hard negatives, 
so there are some works that explore the construction of hard negatives.
The most prominent method 
is based on mixup~\cite{zhang2017mixup},
a data augmentation strategy 
for creating convex linear combinations between samples.
In the area of computer vision, 
Mochi~\cite{kalantidis2020hard} 
measures the distance between samples by inner product 
and randomly selects two samples 
from $N$ nearest ones 
to be combined by mixup as synthetic negative samples.
Further, 
CuCo~\cite{che2021multi} uses cosine similarity to measure the difference of nodes in homogeneous graphs.
In heterogeneous graphs,
STENCIL~\cite{zhu2021structure}
uses meta-path-based Laplacian positional embeddings and personalized PageRank scores for modeling local structural patterns of the meta-path-induced view.
However, these methods either fail to distinguish hard negative samples from false ones 
or 
are built on one type of relation in HINs, which restricts the wide applicability of these models. 

\section{Preliminary}
In this section, 
we formally define some related concepts used in this paper. 

\begin{definition}
\textbf{Heterogeneous Information Network (HIN).}
An HIN is defined as a graph $\mathcal{G} = (\mathcal{V}, \mathcal{E})$, where $\mathcal{V}$ is a set of nodes and $\mathcal{E}$ is a set of edges, each represents a binary relation between two nodes in $\mathcal{V}$.
Further,
$\mathcal{G}$
is associated with two mappings: 
(1) node type mapping function $\phi : \mathcal{V} \to \mathcal{T}$ and (2) edge type mapping function $\psi : \mathcal{E} \to \mathcal{R}$, where $\mathcal{T}$ and $\mathcal{R}$ denote the sets of node and edge types, respectively. 
If $\vert \mathcal{T} \vert + \vert \mathcal{R} \vert > 2$,
the network $\mathcal{G}$ 
is 
an HIN; otherwise it is homogeneous.
\end{definition}

\begin{definition}
\textbf{Meta-path.}
A meta-path $\mathcal{P}$: $T_1 \stackrel{R_1}{\to} T_2 \stackrel{R_2}{\to} \dots \stackrel{R_l}{\to} T_{l+1}$ (abbreviated as $T_1 T_2 \dots T_{l+1}) $ 
expresses a composite relation $R = R_1 \circ R_2 \circ \dots \circ R_l$ between nodes of types $T_1$ and $T_{l+1}$, 
where $\circ$ denotes the composition operator on relations. 
If two nodes $x_i$
and $x_j$ are related by the composite relation $R$, 
then there
exists a path that
connects $x_i$ to $x_j$ in
$\mathcal{G}$, 
denoted by $p_{x_i\rightsquigarrow x_j}$.
Moreover, the sequence of nodes and edges in $p_{x_i\rightsquigarrow x_j}$
matches the sequence of types $T_1, ..., T_{l+1}$ and relations $R_1, ..., R_l$ according to the node type mapping $\phi$ and the edge type mapping $\psi$, respectively. We say that $p_{x_i\rightsquigarrow x_j}$
is a
path instance of $\mathcal{P}$, denoted by $p_{x_i\rightsquigarrow x_j} \vdash \mathcal{P}$.
\end{definition}

\begin{definition}
\textbf{Meta-path Context~\cite{li2021leveraging}.}
Given two objects $x_i$ and $x_j$ that are related by a meta-path $\mathcal{P}$, 
the meta-path context is the set of path instances of $\mathcal{P}$ between $x_i$ and $x_j$.

\end{definition}

\begin{definition}
  \textbf{Heterogeneous Graph Contrastive Learning.}
Given an HIN $\mathcal{G}$,
our task is to learn 
node representations
by constructing 
positive and negative pairs for contrast.
In this paper,
we only focus on 
one type of nodes, 
which are considered as target nodes.  
\end{definition}



\begin{figure*}[ht]
    \centering
    \includegraphics[width=0.95\textwidth]{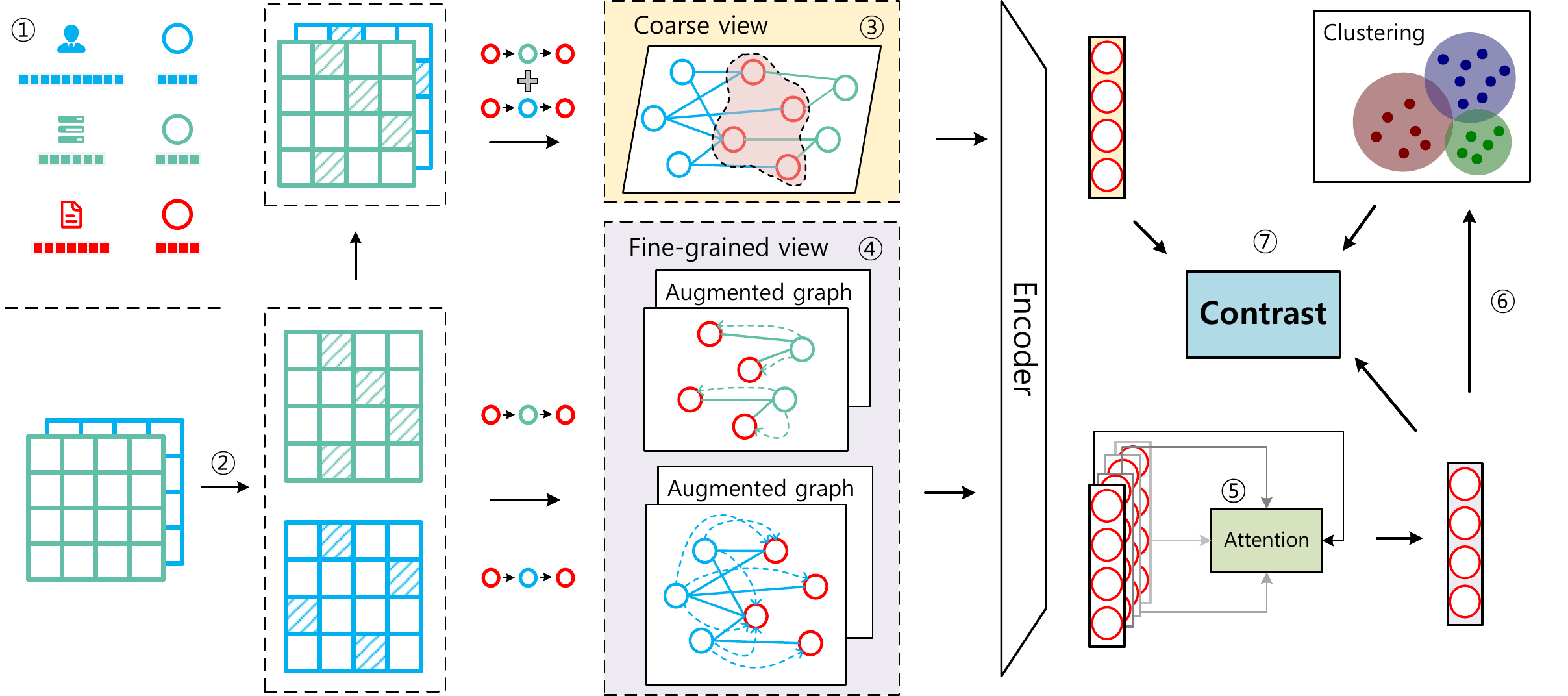}
    \caption{The overall framework of the~\name~model. For details of each step, see Section~\ref{method}.}
    \label{fig:overall}
  \end{figure*}

\section{Methodology}\label{method}
In this section, we introduce our method~\name~and the variant model~\nameplus.
The general model diagram is shown in Fig.~\ref{fig:overall}.
We perform feature transformation and neighbor filtering as preprocessing steps.
First, we map the feature vectors of each different type of nodes into the same dimension~(Step~\ding{172})
and identify a set of neighbors for nodes based on each meta-path~(Step~\ding{173}).
Then, we construct a coarse view by aggregating all meta-paths (Step~\ding{174}), while constructing a fine-grained view with each meta-path's contextual semantic information (Step~\ding{175}).
After that, we fuse different embeddings from various meta-paths in the fine-grained view through the attention mechanism ~(Step~\ding{176}).
We take node embeddings in the coarse view as anchors and those in the fine-grained view as the positive and negative samples.
To be capable of distinguishing false negative samples and hard negative samples, 
we perform clustering and assign weights to the negative samples with the clustering results (Step~\ding{177}).
Finally, 
to further boost the model
performance, 
we use prototypical contrastive learning to calculate the contrastive loss and prototypical loss based on the node embedding vectors under coarse view and fine-grained view and the clustering results (Step~\ding{178}).
{
In addition, to further obtain adaptive negative sample weights, we propose the variant \nameplus~with MLP to learn the weights instead of clustering and calculate the contrastive loss.
}
Next, we describe each component in detail.

\subsection{Node Feature Transformation}
\label{sec:4.1}
Since an HIN is composed of different types of nodes and each type has its own feature space,
we need to first preprocess node features to transform them into the same space.
Specifically, 
for each object $x_i$ in type $T$,
we use the type-specific mapping matrix $W_{T}^{(1)}$ to transform 
the raw features of 
$x_i$ into: 
\begin{equation}
\label{eq:initial_emb}
    h_i = \sigma(W_{T}^{(1)} \cdot x_i + b_{T}),
\end{equation}
where $h_i \in \mathbb{R}^{d}$ is the projected initial embedding vector of 
$x_i$, $\sigma(\cdot)$ is an activation function, and $b_{T}$ denotes the bias vector.


\subsection{Neighbor Filtering}
Given an object $x$,
meta-paths can be used to derive its multi-hop neighbors with specific semantics.
When meta-paths are long,
the number of related neighbors to $x$ could be numerous.
Directly aggregating information from these neighbors to generate $x$'s embedding will be time-consuming.
On the other hand,
the irrelevant neighbors of $x$
cannot provide useful information to predict $x$'s label and they could adversely affect the quality of the generated embedding of $x$.
Therefore,
we filter $x$'s meta-path induced neighbors and select the most relevant to $x$.
Inspired by~\cite{li2021leveraging},
we adopt \emph{PathSim}~\cite{sun2011pathsim} to measure the similarity between objects.
Specifically,
given a meta-path $\mathcal{P}$,
the similarity between two objects $x_i$ and $x_j$ of the same type w.r.t. 
$\mathcal{P}$ is computed by:
\begin{equation}
\small
  PS(x_i,x_j) = \frac{2\times \vert \{p_{x_i \rightsquigarrow x_j} \vert p_{x_i \rightsquigarrow x_j} \vdash \mathcal{P}\} \vert }{ \vert \{p_{x_i \rightsquigarrow x_i} \vert p_{x_i \rightsquigarrow x_i} \vdash \mathcal{P}\} \vert +  \vert \{p_{x_j \rightsquigarrow x_j} \vert p_{x_j \rightsquigarrow x_j} \vdash \mathcal{P}\} \vert},
\end{equation}
where $p_{x_i \rightsquigarrow x_j}$ is a path instance between $x_i$ and $x_j$.
Based on the similarities,
for each object,
we select its top-$K$ neighbors with the largest similarity.
The removal of irrelevant neighbors can significantly reduce the number of neighbors for each object,
which further improves the model efficiency.
After neighbor filtering,
the induced adjacency matrix by meta-path $\mathcal{P}$ 
is denoted as $A^{\mathcal{P}}$.


\subsection{Coarse View}
We next construct coarse view 
to describe which
objects are connected by meta-paths.
Given a set of meta-paths, 
each meta-path 
$\mathcal{P}$ can induce its own adjacency matrix $A^{\mathcal{P}}$.
To provide a coarse view on the connectivity between objects by meta-paths,
we fuse the meta-path induced adjacency matrices and define
$\widetilde{A} = \frac{1}{m} ( \widetilde{A^{\mathcal{P}_1}} + \widetilde{A^{\mathcal{P}_2}} + \cdots +  \widetilde{A^{\mathcal{P}_m}}) $,
where
$m$ is the number of meta-paths and $\widetilde{A^{\mathcal{P}_u}} = D^{-1/2}A^{\mathcal{P}_u}D^{-1/2}$ is the normalized adjacency matrix.
Here,
$D$ is a diagonal matrix with $D_{ii} = \sum_{j=1}^{|V|} A^{\mathcal{P}_u}_{ij}$,
where $|V|$ is the number of target nodes.
After that, 
we feed node embeddings 
calculated by Equation~\ref{eq:initial_emb}
and $\widetilde{A}$ into 
a two-layer GCN encoder to get the representations of nodes in the coarse view.
Specifically,
for node $x_i$,
we can get its coarse representation $z_i^{c}$:
\begin{equation}
    z_i^{c} = {\rm Encoder}(\widetilde{A},h_i),
\label{eq:coarse}
\end{equation}

\subsection{Fine-grained View}
The fine-grained view 
characterizes how two objects are connected by meta-paths,
which is in contrast with the coarse view.
Given a meta-path set $\mathcal{PS} = \{\mathcal{P}_1, ..., \mathcal{P}_m\}$,
for each meta-path $\mathcal{P}_u \in \mathcal{PS}$,
let $\mathcal{P}_u = T_0T_1...T_l$,
where the meta-path length is $l+1$.
The meta-path 
can link objects of type $T_0$ to that of type $T_l$ via a series of intermediate object types.
Since meta-path contexts are composed of path instances and 
capture details on how two objects are connected,
we utilize meta-path contexts
to 
learn fine-grained representations for objects.
However,
when $l$ is large,
due to the numerous path instances between two objects,
directly handling each path instance as MAGNN~\cite{fu2020magnn} could significantly degenerate the model efficiency, as pointed out in~\cite{li2021leveraging}.
We instead use objects in the intermediate types of meta-path $\mathcal{P}_u$
to leverage the information of meta-path contexts.
Specifically,
given a meta-path $\mathcal{P}_u$
and an object $x_i$ of type $T_0$,
we denote $N_i^{T_j}$ as $x_i$'s $j$-hop
neighbor set w.r.t. $\mathcal{P}_u$.
Then we generate $x_i$'s initial fine-grained embedding by aggregating information from all its $j$-hop neighbors with $j\leq l$.
Formally, 
we have
\begin{equation}
    h_i^{\mathcal{P}_u} = \sigma(h_i + \sum_{ j = 1 }^{l} \sum_{ x_v \in N_i^{T_{j}} } W_{uj}^{(2)} \cdot h_v ),
\end{equation}
where the learnable parameter matrix $W^{(2)}_{uj}$ 
corresponds to
the $j$-hop neighbors w.r.t. $\mathcal{P}_u$.
After that, 
we put the node embedding $h_i^{\mathcal{P}_u}$ that aggregates the meta-path context information 
and the adjacency matrix under the meta-path $\widetilde{A^{\mathcal{P}_u}}$ into a two-layer GCN encoder to generate $x_i$'s fine-grained embedding:
\begin{equation}
\label{eq:fgz}
    z_i^{\mathcal{P}_u} = {\rm Encoder} (\widetilde{A^{\mathcal{P}_u}},h_i^{\mathcal{P}_u}),
\end{equation}
Note that the encoder here is the same as that used in the coarse view (see Equation~\ref{eq:coarse}).
Further, 
to improve the model generalizability,
we introduce noise to the meta-path induced graph by performing graph augmentation, such as edge masking and feature masking.
After the perturbed graph is generated,
we feed it into Equation~\ref{eq:fgz}
to
generate the node embedding $z_i^{\hat{\mathcal{P}_u}}$.
In this way,
for each meta-path $\mathcal{P}_u$ and an object $x_i$,
we generate two embeddings 
$z_i^{\mathcal{P}_u}, z_i^{\hat{\mathcal{P}_u}}$.
Given a meta-path set $\mathcal{PS} = \{\mathcal{P}_1, ..., \mathcal{P}_m\}$,
we can generate ${Z}_i=\{ z_i^{\mathcal{P}_u}, z_i^{\hat{\mathcal{P}_u}} \vert \mathcal{P}_u \in \mathcal{PS}  \}$ 
for node $x_i$ from various meta-paths.
Finally,
we fuse these embeddings by the attention mechanism:
\begin{equation}
    w_s = \frac{1}{\vert V \vert} \sum_{x_i \in V} \textbf{a}^{\rm T} \cdot \tanh(W_{att}z_i^{s} + b_{att}),
    \;
    \beta_s = \frac{\exp (w_s)}{\sum\nolimits_{j=1}^{\vert Z_i \vert} \exp(w_j)}
\end{equation}

Here, we measure the weight of each node type. 
$V$ is the set of target nodes, $W_{att} \in \mathbb{R}^{d \times d}$ is the weight matrix, $b_{att}$ is the bias vector and $\beta_s$ denotes the attention weight.
We can generate $x_i$'s fine-grained embedding vector $z_{i}^{f}$:
\begin{equation}
    z_{i}^{f} = \sum_{s=1}^{\vert Z_i \vert} \beta_s \cdot z_i^s, 
\end{equation}

\subsection{Theoretical analysis on the InfoNCE loss}

{
In contrastive learning, 
different negative samples have different characteristics,
so their impact should not be the same. 
For a given anchor, 
some negative samples are easy to distinguish, while some hard negative samples may have a certain degree of similarity with the anchor but belong to a different class.
Therefore, 
in order to keep negative samples away from the anchor,
it is necessary to distinguish the effects of different negative samples on the anchor. 
With this in mind, 
we first propose Theorem~\ref{theorem1}.

\begin{theorem}
\label{theorem1}
Consider 
the contrastive learning InfoNCE loss~\cite{oord2018representation} that uses dot product to measure node similarity,
denoted as $\mathcal{L}$.
Let $f(x)$ represent the learned embedding of node $x$.
Given 
$x_i$ as an anchor,
$x_k$ as its positive sample
and $x_{t_1}, x_{t_2}$ as its two negative samples, 
with back propagation,
we can get:
(1)
If $f(x_i)^Tf(x_{t_1}) > f(x_i)^Tf(x_{t_2})$, 
then 
$ 
\bigg\vert 
\frac{\partial \mathcal{L}}{\partial f(x)}
\vert_{x=x_{t_1}}
\bigg\vert 
>
\bigg\vert 
\frac{\partial \mathcal{L}}{\partial f(x)}
\vert_{x=x_{t_2}}
\bigg\vert 
$,
(2)
$ 
\bigg\vert 
\frac{\partial \mathcal{L}}{\partial f(x)}
\vert_{x={x_k}}
\bigg\vert 
\geq
\bigg\vert 
\frac{\partial \mathcal{L}}{\partial f(x)}
\vert_{x=x_{t_1}}
\bigg\vert
$.
\end{theorem}

\begin{proof}

The contrastive loss function InfoNCE is defined as:
\begin{equation}
\begin{aligned}
\mathcal{L}_i
&= -\log \frac{\exp(\text{sim}(f(x_i), f(x_k)) / \tau)}{\sum_{j=1}^n \exp(\text{sim}(f(x_i), f(x_j)) / \tau)}, \\
\end{aligned}
\end{equation}
where 
$\text{sim}(f(x_i), f(x_j))$ measures the similarity between node embeddings $f(x_i)$ and $f(x_j)$,
$\tau$ is a hyperparamter denotes the temperature and $n$ is the number of negative samples.
Typically, the dot product is used as a similarity function, and the InfoNCE loss can be further simplified as:
\begin{equation}
\begin{aligned}
\mathcal{L}_i
&= -\log \frac{\exp(f(x_i)^Tf(x_k) / \tau)}{\sum_{j=1}^n \exp(f(x_i)^Tf(x_j) / \tau
)} \\
&=
-\log \exp(f(x_i)^Tf(x_k) 
\frac{1}{\tau}
) +
\log [\sum_{j=1}^n \exp(f(x_i)^Tf(x_j)
\frac{1}{\tau}
)] \\
&=
-f(x_i)^Tf(x_k) 
\frac{1}{\tau}
+
\log [ \sum_{j=1}^n \exp(f(x_i)^Tf(x_j) 
\frac{1}{\tau}
) ]
\end{aligned}
\end{equation}

For a particular negative sample $x_t, t= 1,2,\cdots,n$,
the gradient of the negative sample $x_t$ is:
\begin{equation}
\begin{aligned}
&\quad \frac{\partial \mathcal{L}_i}{\partial f(x)} \bigg\vert_{x=x_t} \\
&=
\frac{1}
{ \sum_{j=1}^n \exp(f(x_i)^Tf(x_j) / \tau)}
\frac{\partial \sum_{j=1}^n \exp(f(x_i)^Tf(x_j) / \tau)}
{\partial f(x_t)} \\
&=
\frac{\exp(f(x_i)^Tf(x_t) / \tau)}
{ \sum_{j=1}^n \exp(f(x_i)^Tf(x_j) / \tau)}
\frac{\partial (f(x_i)^Tf(x_t) / \tau)}
{\partial f(x_t)} \\
&=
\frac{\exp(f(x_i)^Tf(x_t) / \tau)}
{ \sum_{j=1}^n \exp(f(x_i)^Tf(x_j) / \tau)}
\frac{1}{\tau}
f(x_i) \\
\end{aligned}
\label{eq:neg}
\end{equation}

For all the negative samples 
of anchor $x_i$,
the gradient only depends on $f(x_i)^Tf(x_t)$.
This is 
because $f(x_i)$ determines the direction of back propagation, 
and $\frac{1}
{ \sum_{j=1}^n \exp(f(x_i)^Tf(x_j) / \tau)}
\frac{1}{\tau}$ is equal for all the negative samples.
We can thus derive inequality (1) in Theorem~\ref{theorem1}, which states that 
if $f(x_i)^Tf(x_{t_1}) > f(x_i)^Tf(x_{t_2})$, 
then 
$ 
\bigg\vert 
\frac{\partial \mathcal{L}_i}{\partial f(x)}
\vert_{x=x_{t_1}}
\bigg\vert 
>
\bigg\vert 
\frac{\partial \mathcal{L}_i}{\partial f(x)}
\vert_{x=x_{t_2}}
\bigg\vert 
$.


In addition, 
we can also compute the gradient of the positive sample by taking its derivative.

\begin{equation}
\label{eq:pos}
\frac{\partial \mathcal{L}_i}{\partial f(x)} \bigg\vert_{x=x_k} 
=
-
\frac{1}{\tau}
f(x_i)
\end{equation}
We observe that compared to positive samples,
Equation~\ref{eq:neg} for negative samples has an 
additional softmax term, whose
value ranges between 0 and 1.
So we can derive inequality (2) in Theorem~\ref{theorem1}:
$\bigg\vert \frac{\partial \mathcal{L}_i}{\partial f(x)} \vert_{x=x_k} \bigg\vert \geq \bigg\vert \frac{\partial \mathcal{L}_i}{\partial f(x)} \vert_{x=x_t} \bigg\vert$.
The equation holds if and only if the softmax term equals one, which is generally very difficult to satisfy.


\end{proof}

From Theorem~\ref{theorem1},
easy negative samples 
that are less similar to the anchor 
lead to smaller gradient magnitude, 
while 
hard negative samples 
can derive larger gradient magnitude.
This is because easy negative samples are already far enough from the anchor and we don't need to pay much attention to them. 
However, 
hard negative samples need larger gradients to push them apart.
Further,
the comparison between Equation~\ref{eq:neg} and 
Equation~\ref{eq:pos} shows that
the gradient magnitude of positive samples is generally much larger than that of negative samples,
due to the additional softmax term 
that is generally smaller than 1
in
Equation~\ref{eq:neg}.
In summary,
in each epoch,
compared to negative samples with lower similarity to the anchor, 
negative ones with higher similarity will be more violently pushed 
away from the anchor.
On the other hand,
positive samples will have a larger update magnitude than negative samples, 
resulting in a closer proximity to the anchor.


We next randomly select a paper and an author as the anchor node from the ACM dataset~\cite{zhao2020network} and the DBLP dataset~\cite{fu2020magnn}, respectively,
and study the relationship between node similarity and gradient magnitude of loss functions w.r.t. negative samples.
As shown in
Figure~\ref{fig:con},
the orange curves in both sub-figures
show that 
the gradient magnitude of the InfoNCE loss is proportionally to 
the similarity between negative samples with the anchor node,
which is consistent with Equation~\ref{eq:neg}. 
Although the InfoNCE loss can distinguish 
samples from different classes to some extent,
the gradient magnitude only depends on node similarity,
which lacks the flexibility to 
capture the variability in node embeddings.
For example, 
suppose there are three samples whose representations are: $x_1 (1,1)$, $x_2 (1,0)$ and $x_3 (0,1)$, respectively.
We take
$x_1$ 
as the anchor.
For the sample pairs $(x_1,x_2)$ and $(x_1,x_3)$, 
their similarity values are both 1, 
but the semantic information contained in $x_2$ and $x_3$
is completely different, 
and even opposite to each other.
This should further lead to 
different gradient update directions.
Therefore, 
using only node similarity to determine the gradient of a negative sample is insufficient. 
We thus need to introduce other
metrics to capture the fine-grained information of embeddings of negative samples.
}

\subsection{The \name\ model}
In this section,
we perform contrastive learning to learn node embeddings with the constructed coarse view and fine-grained view and propose our loss function with additional weights for negative sample pairs.
Before contrast, 
we use a projection head ({one-layer MLP}) to map node embedding vectors to the space where contrastive loss can be applied.
Specifically,
for $x_i$, we have:
\begin{equation}
\begin{aligned}
\overline{z}_i^{c} &= \sigma (W_{proj} z_i^{c} + b_{proj} ),\\ 
\overline{z}_i^{f} &= \sigma (W_{proj} z_i^{f} + b_{proj} ),
\end{aligned}
\end{equation}
After that,
we take representations in the coarse view as anchors and
construct the positive and negative samples from the fine-grained view.
For each node $x_i$,
we take $\overline{z}_i^{c}$ as the anchor, $\overline{z}_i^{f}$ as the corresponding positive sample, and 
all other node representations in the fine-grained view
as negative samples.
Further,
to utilize hard negatives and mitigate the adverse effect of false negatives,
we learn the importance of negative samples.
In particular,
we perform node clustering 
based on the fine-grained representations for 
$M$  times, 
where the number of clusters are set as 
$U = \{k_1, k_2, \cdots, k_M\}$.
Then, we assign different weights
to negative samples of a node based on the clustering results.
Intuitively,
when the number of clusters is set large,
each cluster will become compact.
Then
compared with hard negatives,
false negatives and easy negatives are more likely to be assigned in the same cluster and different clusters with the anchor node, respectively. 
Therefore,
we use $\gamma_{ij}$ to denote the weight of node $x_j$ as a negative sample to node $x_i$ and set 
it as a function $\mathcal{F}$ of clustering results.

\begin{figure}
    \centering
\includegraphics[width=0.46\textwidth]{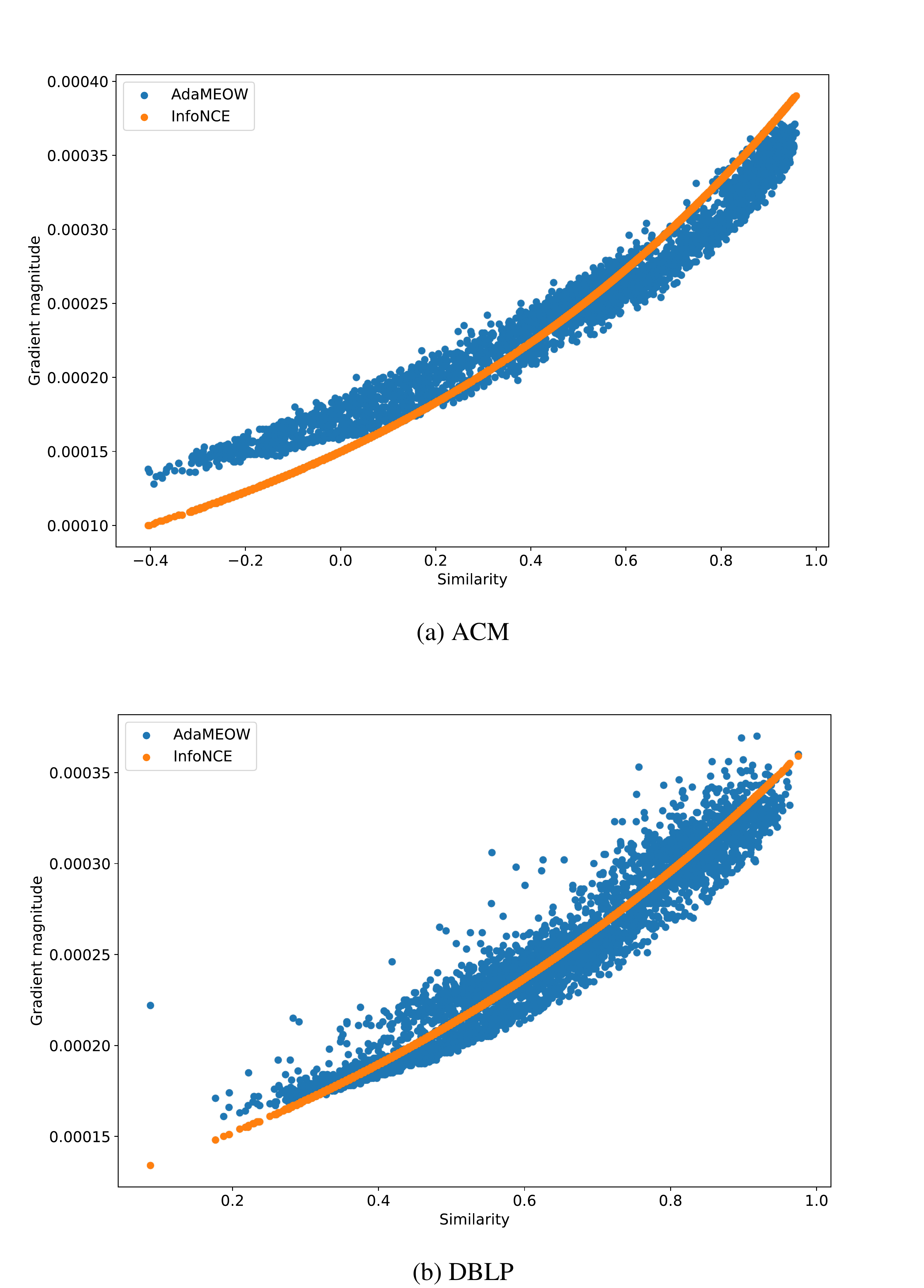}
    \caption{
    {
    The relationship between node similarity $\text{sim}(\cdot)$ with a randomly selected anchor and gradient magnitude of loss functions w.r.t. negative samples after training 500 epochs on the (a) ACM dataset and 800 epochs on the (b) DBLP dataset.
    The orange dots indicate the InfoNCE loss and the blue dots indicate the loss function adopted by AdaMEOW.
    }
    }
    \label{fig:con}
  \end{figure}

For simplicity,
we define the function $\mathcal{F}$ to count the number of times that the sample $x_j$ and the anchor $x_i$ are in different clusters.
We denote $\gamma_{ij} = \mathcal{F}_{ij}(C_{1},C_{2},\cdots,C_{M})$,
where $C_{r}$ is the $r$-th clustering result.
In particular,
we can understand
$\gamma_{ij}$ as the push strength.
For false negatives,
$\gamma_{ij}$ should be small to ensure that they will not be pushed away from the anchor.
For hard negatives,
$\gamma_{ij}$ is expected to be much larger because 
in this way, the anchor and hard negatives can be discriminated.
Since easy samples are distant from the anchor,
the model will be insensitive to
$\gamma_{ij}$ 
in a wide range of values.
Then
based on $\gamma_{ij}$,
we can formulate our contrastive loss function as
\begin{equation}
    \mathcal{L}^{con}_i = - \log \frac{ {\rm exp}(\overline{z}_i^{c} \cdot \overline{z}_i^{f}/ \tau)}
    {\sum\nolimits_{j=1}^{\vert V \vert} \gamma_{ij} \, {\rm exp}(\overline{z}^{c}_i \cdot \overline{z}_j^{f}/ \tau)}
    \label{eq:meow_con}
\end{equation}
where $\tau$ is a temperature parameter.

{
Similar as Theorem~\ref{theorem1},
we also analyze our loss function from the perspective of gradient in the back propagation process.

\begin{theorem}
    For the proposed loss function in Equation~\ref{eq:meow_con},
    we use node embedding function $f(\cdot)$ to overload $\bar{z}$.
    Then given an anchor node $x_i$ 
    with positive sample
    $x_k$ and one of its
    negative sample $x_t$,
    the gradient of $\mathcal{L}^{con}_i$ w.r.t. $f(x)$ at $x_t$
    is 
    \begin{equation}
    \nonumber
    \frac{\partial \mathcal{L}^{con}_i}{\partial f(x)} \bigg\vert_{x=x_t} = \frac{\exp(f(x_i)^Tf(x_t) / \tau)}
{ \sum_{j=1}^n \gamma_{ij} \exp(f(x_i)^Tf(x_j) / \tau)} 
\frac{\gamma_{it}}{\tau}
f(x_i)
\end{equation}
\end{theorem}

\begin{proof}
Similar as the InfoNCE loss, 
our proposed loss function $\mathcal{L}^{con}_i$ for anchor $x_i$ can be simplified as:
\begin{equation}
\begin{aligned}
\mathcal{L}^{con}_i
&=
-f(x_i)^Tf(x_k) / \tau
+
\log [ \sum_{j=1}^n \gamma_{ij} \exp(f(x_i)^Tf(x_j) / \tau) ]
\end{aligned}
\end{equation}
The gradient of the positive sample is the same as the InfoNCE loss, 
and the derivative of the representation of the negative sample $x_t$ can be obtained as:
\begin{equation}
\begin{aligned}
\frac{\partial \mathcal{L}^{con}_i}{\partial f(x)} \bigg\vert_{x=x_t}
&=
\frac{\exp(f(x_i)^Tf(x_t) / \tau)}
{ \sum_{j=1}^n \gamma_{ij} \exp(f(x_i)^Tf(x_j) / \tau)} 
\frac{\gamma_{it}}{\tau}
f(x_i) \\
\end{aligned}
\label{eq:theorem2}
\end{equation}
The gradient magnitude now has an additional learnable parameter $\gamma_{it}$ for $x_t$, 
which assigns personalized weights to negative samples sharing the same similarity with the anchor.

\end{proof}

}

Compared to the original InfoNCE loss, 
our proposed loss function relies not only on the similarity between anchor and negative samples, 
but also on the characterization of anchor and negative samples during the optimization process.
This can be further combined with the characterization of node pairs to adaptively adjust the push strength in the hidden space, 
thus improving the quality of the representation.
For example, 
negative samples  $x_2 (1,0)$ and $x_3 (0,1)$ have the same similarity values with anchor $x_1 (1,1)$. The learnable weights $\gamma_{12}$ and $\gamma_{13}$ makes them more distinguishable,
and the gradients of the two are different during backpropagation.

To further 
make embeddings of nodes in the same 
cluster more compactly distributed in the latent space,
we introduce an additional prototypical contrastive learning loss function.
In the $r$-th clustering,
we consider the prototype vector $c_i^r$, i.e., the cluster center, 
corresponding to node $x_i$ as a positive sample and other prototype vectors as negative samples
and define:
\begin{equation}
    \mathcal{L}_i^{proto} = - \frac{1}{M} \sum_{r=1}^{M} \log \frac{\exp(\overline{z}_i^{c} \cdot c_i^r / \theta_i^r)}{\sum\nolimits_{j=1}^{k_r} \exp(\overline{z}_i^{c} \cdot c_j^r / \theta_j^r) }
    \label{eq:meow_pro_con}
\end{equation}
where 
$\theta_i^r$ is a temperature parameter and
represents the concentration estimate of the cluster $C^i_r$ that contains node $x_i$. 
Following~\cite{li2020prototypical},
we calculate
$\theta_i^r = \frac{\sum_{q=1}^Q\Vert \overline{z}_q^c - c^r_i \Vert _2}{Q\log(Q+\alpha)}$, 
where $Q$ is the number of nodes in the cluster
and $\alpha$ is a smoothing parameter to ensure that small clusters do not have an overly-large $\theta$. 
Finally, 
we formulate our objective function
$\mathcal{L}$ as:
\begin{equation}
    \mathcal{L}_{\name} = \frac{1}{\vert V \vert} \sum_{x_i \in V}(\mathcal{L}^{con}_i +\lambda \mathcal{L}^{proto}_i)
    \label{meow_loss}
\end{equation}
where
$V$ is the set of target nodes and 
$\lambda$ controls the relative importance of the two terms.
The loss function can be optimized by stochastic gradient descent. 
To prevent overfitting,
we further regularize all the weight matrices $W$ mentioned above.{
The whole training procedure of the \name\ model is summarized in Algorithm \ref{alg:meow}.
}

{

\subsection{The \nameplus\ model}

\begin{figure}
    \centering
\includegraphics[width=0.4\textwidth]{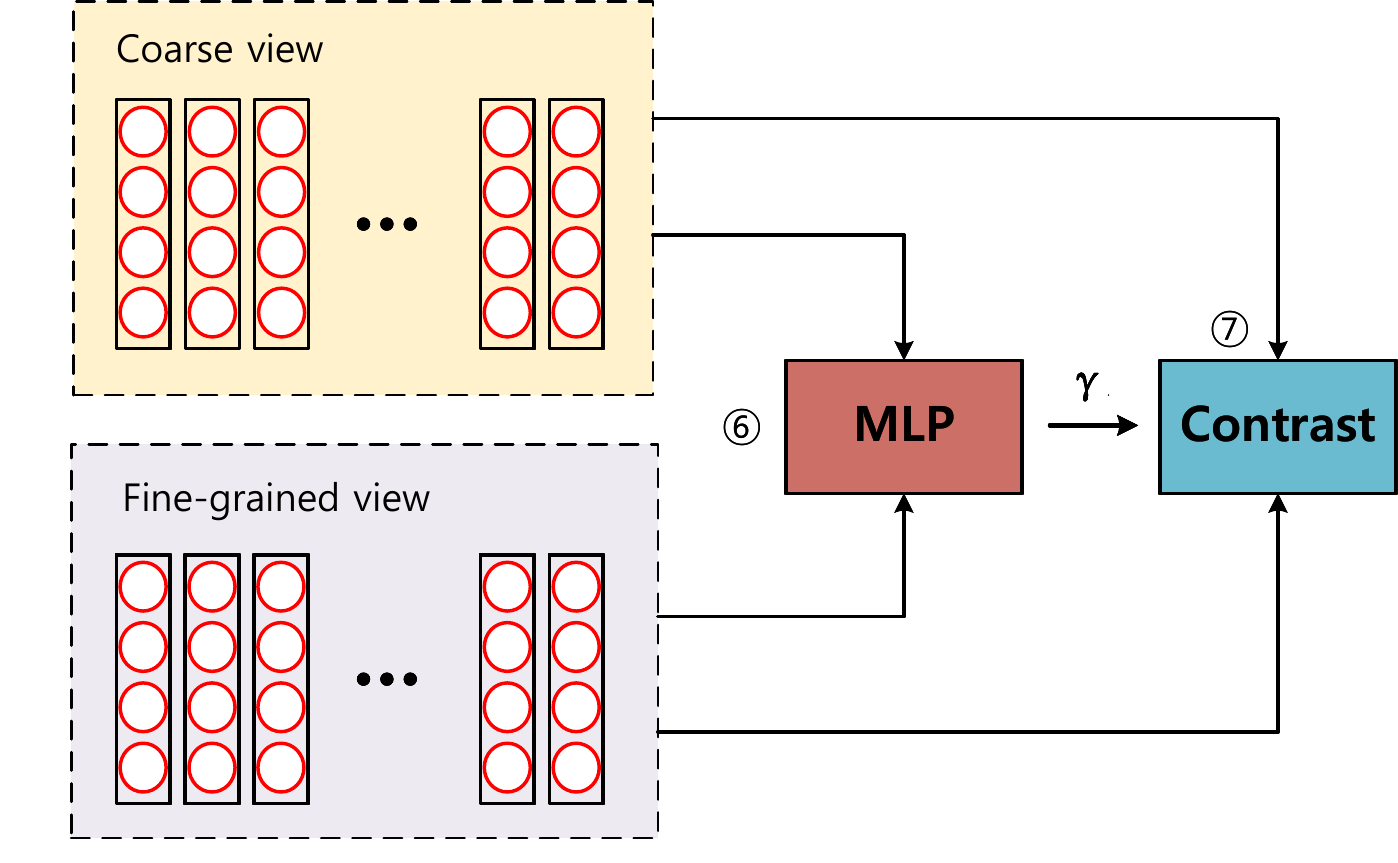}
    \caption{
    {
    The contrastive part of the~\nameplus~model.
    }
    }
    \label{fig:model_2}
  \end{figure}

The weights we calculate in \name~based on the clustering results 
can actually provide 
more fine-grained differentiations of negative samples.
However, 
these weights are hard values and could limit node representation learning.
Therefore, 
it is necessary to explore alternative approaches that can better facilitate flexible weight calculations, 
thereby enhancing the overall performance of the model.
Instead of clustering,
we propose an enhanced model \nameplus\ with an adaptive approach to learn soft-valued weights.
Specifically,
we apply a \emph{two-layer MLP} to learn the weights according to the node representations under both the coarse and fined-grained views,
which is formally formulated as:
\begin{equation}
    \widetilde{\gamma_{ij}} = 
    \sigma^{(2)}
    (W_{Ada}^{(2)}(\sigma^{(1)}
    (W_{Ada}^{(1)}\mathcal{H}  (\overline{z}_i^{c}, \overline{z}_j^{f}) 
    + b_{Ada}^{(1)}))
    + b_{Ada}^{(2)}),
\end{equation}
where $W_{Ada}^{(1)},W_{Ada}^{(2)}$ are the learnable parameter matrices and
$b_{Ada}^{(1)},b_{Ada}^{(2)}$ are the bias vectors.
Note that
$\sigma^{(1)}$ is the Tanh function and $\sigma^{(2)}$ is the sigmoid function,
which
ensures that the weights are soft values ranging from 0 and 1.
$\mathcal{H}  (\overline{z}_i^{c}, \overline{z}_j^{f})$ is a pooling function between $\overline{z}_i^{c}$ under the coarse view and $\overline{z}_i^{f}$ under the fine-grained view,
and
we use SUM as the pooling function.
The overall objective 
is given by:
\begin{equation}
    \mathcal{L}_{Ada}
    = \frac{1}{\vert V \vert} \sum_{x_i \in V} - \log \frac{ {\rm exp}(\overline{z}_i^{c} \cdot \overline{z}_i^{f}/ \tau)}
    {\sum\nolimits_{j=1}^{\vert V \vert} \widetilde{\gamma_{ij}} \, {\rm exp}(\overline{z}^{c}_i \cdot \overline{z}_j^{f}/ \tau)},
    \label{eq:ada}
\end{equation}
where $\widetilde{\gamma_{ij}}$ is a soft-valued weight between anchor $x_i$ and negative sample $x_j$.
We distinguish it from the hard-valued weight by using the tilde $(\thicksim)$ symbol.
For each anchor node,
the two-layer MLP
can adaptively learn the weights of negative samples,
and further lead to 
more informative gradients for them.
As shown by the blue dots in Figure~\ref{fig:con}, 
optimizing Equation~\ref{eq:ada} allows for a more diverse set of gradient magnitudes for the same node similarity values.
This further shows that 
$\widetilde{\gamma_{ij}}$
can capture the individual characteristics of different negative samples
and their corresponding gradients are not only determined by the similarity with the anchor.
The contrastive part of the \nameplus\ model is summarized in Algorithm \ref{alg:model}.
}

\subsection{Complexity analysis}
The major time complexity in our proposed model comes from GCN and MLP.
Let $d_{max}$ be the maximum initial dimensions of different types of nodes and $d_A$ be the average number of non-zero entries in each row of the adjacency matrix for each meta-path induced graph.
In Section~\ref{sec:4.1},
The time complexity for MLP is $O(Bd_{max}d)$ where $B$ denotes the batch size and $d$ is the dimension of the projected initial embedding vector.
The GCN encoder used in the construction of the two views has a time complexity of $O(Bd_Ad+Bd^2)$.
After constructing the views,
the time complexity of the contrastive loss function is $O(B^2d)$.
For~\name,
node clustering requires $O(t(k_1 + k_2 \dots + k_M)|V|d)$,
where $t$ is the iteration times and node clustering times $M\le2$
in our experiments.
For~\nameplus,
there is an additional MLP with time complexity $O(B^2d^2)$.

\begin{algorithm}[htb] 
\caption{ The \name\ model} 
\label{alg:meow}
\begin{algorithmic}[1]

{
\REQUIRE ~~\\
    {The heterogeneous graph $\mathcal{G} = (\mathcal{V},\mathcal{E})$; 
    the number of node type $\vert \mathcal{T} \vert$;
    the number of target node $|V|$;
    the feature matrix $X_1,X_2,\cdots,X_{\vert \mathcal{T} \vert}$;
    a pre-defined meta-path set $\mathcal{PS}$;
    the number of clusters $U=\{k_1,k_2,\cdots,k_M\}$;}
\ENSURE ~~\\
    Target node embeddings for downstream tasks.
    
 \STATE  $//$ Pre-Process.
 
\FORALL{$\mathcal{P}_u \in \mathcal{PS} $} 
\STATE{Calculate $\mathcal{P}_u$-induced \emph{PathSim} scores based on Eq. 2};
\STATE{Filter low-impact neighbors and obtain $A^{\mathcal{P}_u}$};
\STATE{Compute the normalized adjacency matrix: $\widetilde{A^{\mathcal{P}_u}} = D^{-1/2}A^{\mathcal{P}_u}D^{-1/2}
$};
\ENDFOR
\STATE Calclaute $
\widetilde{A} = \frac{1}{m} ( \widetilde{A^{\mathcal{P}_1}} + \widetilde{A^{\mathcal{P}_2}} + \cdots +  \widetilde{A^{\mathcal{P}_m}}) $;

 \STATE  $\triangleright$ Lines 9-23 correspond to one epoch
\STATE Transform node feature and get $h_i$ in coarse view using Eq. 1;

 \STATE  $//$ Coarse view

\STATE Calculate representations $z_i^c$ in coarse view using Eq. 3;

 \STATE  $//$ Fine-grained view

\STATE Aggregate meta-contexts to $h_i^{P_u}$ under a meta-path using Eq. 4;

\STATE Calculate representations $z_i^{P_u}$ under a meta-path using Eq. 5;

\STATE Use attention mechanism to generate representations $z_i^f$ under fine-grained view using Eq. 6;

\STATE Map the representations under two views to contrastive space and get $ \overline{z}_i^{c},  \overline{z}_i^{f} $ using Eq. 7;

 \STATE  $//$ Contrastive part

\FOR{ $ r = 1$ \TO $ M$ } 
    \STATE Cluster embedding $\overline{Z}^{f}$ under fine-grained view into $k_r$ clusters and get result $C_r$;
    \STATE Calculate temperature parameter $\theta^r$;
\ENDFOR

\STATE Construct $\mathcal{L}^{con}, \mathcal{L}^{proto},\mathcal{L}$ using Eq. 10-12;

\STATE Optimize $\mathcal{L}$ to update all parameters in the model.

\RETURN $\{z_i^f\}^{|V|}_{i=1} $;                
}
\end{algorithmic} 

\end{algorithm}

\begin{algorithm}[htb] 
\caption{The \nameplus\ model}   
\label{alg:model}
\begin{algorithmic}[1]
{
\REQUIRE ~~\\
    {The heterogeneous graph $\mathcal{G} = (\mathcal{V},\mathcal{E})$; 
    the number of node type $\vert \mathcal{T} \vert$;
    the number of target node $|V|$;
    the feature matrix $X_1,X_2,\cdots,X_{\vert \mathcal{T} \vert}$;
    a pre-defined meta-path set $\mathcal{PS}$;}
\ENSURE ~~\\
    Target node embeddings for downstream tasks.

\STATE $\triangleright$  The same steps as in Algorithm 1, from line 1 to line 16;
 \STATE  $//$ Contrastive part

 \STATE  Calculate the weights of negative sample pairs $\gamma_{ij}$ using Eq. 13;
 \STATE  Construct $\mathcal{L}$ based on weights using Eq. 14;
 \STATE Optimize $\mathcal{L}$ to update all parameters in the model.
\RETURN $\{z_i^f\}^{|V|}_{i=1} $; 
}
\end{algorithmic} 
\end{algorithm} 
\section{Experiments}
\subsection{Datasets}
To evaluate the performance of \name,
we employ four real-world datasets: ACM~\cite{zhao2020network}, DBLP~\cite{fu2020magnn}, Aminer~\cite{hu2019adversarial} and IMDB~\cite{luo2021detecting}.
The four datasets are benchmark HINs.
We next define a classification task for each dataset.

\noindent{\small$\bullet$} \textbf{ACM:}
ACM is an academic paper dataset.
The dataset contains 4019 papers (P), 7167 authors (A), and 60 subjects (S).
Links include P-A (an author publishes a paper) and P-S (a paper is based on a subject).
We use PAP and PSP as meta-paths.
Paper features are the bag-of-words representation of keywords. 
Our task is to classify papers into three areas: database, wireless communication, and data mining.

\noindent{\small$\bullet$} \textbf{DBLP:}
DBLP is extracted from the computer science bibliography website.
The dataset contains 4057 authors (A), 14328 papers (P), 20 conferences (C) and 7723 terms (T).
Links include A-P (an author publishes a paper), P-T (a paper contains a term) and P-C (a paper is published on a conference).
We consider the meta-path set $\{$APA, APCPA, APTPA$\}$.
Each author is described by a
bag-of-words vector of their paper keywords. 
Our task is to classify authors into four research areas: 
Database, Data Mining, Artificial Intelligence and Information Retrieval.

\noindent{\small$\bullet$} \textbf{AMiner:}
Aminer is a bibliographic graphs.
The dataset contains 6564 papers (P), 13329 authors (A) and 35890 references (R).
Links include P-A (an author publishes a paper) and P-R (a reference for a paper).
We consider the meta-path set $\{$PAP, PRP$\}$.
Our task is to classify papers into four research areas.

\begin{table*}[t]
  \caption{Quantitative results (\%$\pm\sigma$) on node classification. We highlight the best score on each dataset in bold and the runner-up score with underline.}
  \label{classification}
  \resizebox{\textwidth}{!}{
  \begin{tabular}{c|c|c|ccccccccc|cc}
    \hline
    Datasets & Metric & Split & GraphSAGE & GAE & Mp2vec & HERec & HetGNN & HAN & DGI & DMGI & HeCo & \name & \nameplus \\
    \hline
    \multirow{9}{*}{ACM}&
    \multirow{3}{*}{Ma-F1}
    &20&47.13$\pm$4.7&62.72$\pm$3.1&51.91$\pm$0.9&55.13$\pm$1.5&72.11$\pm$0.9&85.66$\pm$2.1&79.27$\pm$3.8&87.86$\pm$0.2&88.56$\pm$0.8&
    \underline{91.93$\pm$0.3}&
    \yjx{\textbf{93.44$\pm$0.2}}\\
    &&40&55.96$\pm$6.8&61.61$\pm$3.2&62.41$\pm$0.6&61.21$\pm$0.8&72.02$\pm$0.4&87.47$\pm$1.1&80.23$\pm$3.3&86.23$\pm$0.8&87.61$\pm$0.5&
    \underline{91.35$\pm$0.3}&
    \yjx{\textbf{92.36$\pm$0.2}}\\
    &&60&56.59$\pm$5.7&61.67$\pm$2.9&61.13$\pm$0.4&64.35$\pm$0.8&74.33$\pm$0.6&88.41$\pm$1.1&80.03$\pm$3.3&87.97$\pm$0.4&89.04$\pm$0.5&\underline{92.10$\pm$0.3}&
    \yjx{\textbf{93.34$\pm$0.1}}\\
    \cline{2-14}
    &\multirow{3}{*}{Mi-F1}
    &20&49.72$\pm$5.5&68.02$\pm$1.9&53.13$\pm$0.9&57.47$\pm$1.5&71.89$\pm$1.1&85.11$\pm$2.2&79.63$\pm$3.5&87.60$\pm$0.8&88.13$\pm$0.8&\underline{91.82$\pm$0.3}&
    \yjx{\textbf{93.31$\pm$0.2}}\\
    &&40&60.98$\pm$3.5&66.38$\pm$1.9&64.43$\pm$0.6&62.62$\pm$0.9&74.46$\pm$0.8&87.21$\pm$1.2&80.41$\pm$3.0&86.02$\pm$0.9&87.45$\pm$0.5&\underline{91.33$\pm$0.3}&
    \yjx{\textbf{92.50$\pm$0.3}}\\
    &&60&60.72$\pm$4.3&65.71$\pm$2.2&62.72$\pm$0.3&65.15$\pm$0.9&76.08$\pm$0.7&88.10$\pm$1.2&80.15$\pm$3.2&87.82$\pm$0.5&88.71$\pm$0.5&
    \underline{91.99$\pm$0.3}&
    \yjx{\textbf{93.24$\pm$0.2}}\\
    \cline{2-14}
    &\multirow{3}{*}{AUC}
    &20&65.88$\pm$3.7&79.50$\pm$2.4&71.66$\pm$0.7&75.44$\pm$1.3&84.36$\pm$1.0&93.47$\pm$1.5&91.47$\pm$2.3&96.72$\pm$0.3&96.49$\pm$0.3&\underline{98.43$\pm$0.2}&
    \yjx{\textbf{99.00$\pm$0.0}}\\
    &&40&71.06$\pm$5.2&79.14$\pm$2.5&80.48$\pm$0.4&79.84$\pm$0.5&85.01$\pm$0.6&94.84$\pm$0.9&91.52$\pm$2.3&96.35$\pm$0.3&96.40$\pm$0.4&\underline{97.94$\pm$0.1}&
    \yjx{\textbf{98.64$\pm$0.0}}\\
    &&60&70.45$\pm$6.2&77.90$\pm$2.8&79.33$\pm$0.4&81.64$\pm$0.7&87.64$\pm$0.7&94.68$\pm$1.4&91.41$\pm$1.9&96.79$\pm$0.2&96.55$\pm$0.3&\underline{98.40$\pm$0.2}&
    \yjx{\textbf{98.60$\pm$0.1}}\\
    \hline
    \multirow{9}{*}{DBLP}&
    \multirow{3}{*}{Ma-F1}
    &20&71.97$\pm$8.4&90.90$\pm$0.1&88.98$\pm$0.2&89.57$\pm$0.4&89.51$\pm$1.1&89.31$\pm$0.9&87.93$\pm$2.4&89.94$\pm$0.4&91.28$\pm$0.2&\underline{92.57$\pm$0.4}&
    \yjx{\textbf{93.47$\pm$0.2}}\\
    &&40&73.69$\pm$8.4&89.60$\pm$0.3&88.68$\pm$0.2&89.73$\pm$0.4&88.61$\pm$0.8&88.87$\pm$1.0&88.62$\pm$0.6&89.25$\pm$0.4&90.34$\pm$0.3&\underline{91.47$\pm$0.2}&
    \yjx{\textbf{92.37$\pm$0.2}}\\
    &&60&73.86$\pm$8.1&90.08$\pm$0.2&90.25$\pm$0.1&90.18$\pm$0.3&89.56$\pm$0.5&89.20$\pm$0.8&89.19$\pm$0.9&89.46$\pm$0.6&90.64$\pm$0.3&\underline{93.49$\pm$0.2}&
    \yjx{\textbf{94.00$\pm$0.2}}\\
    \cline{2-14}
    &\multirow{3}{*}{Mi-F1}
    &20&71.44$\pm$8.7&91.55$\pm$0.1&89.67$\pm$0.1&90.24$\pm$0.4&90.11$\pm$1.0&90.16$\pm$0.9&88.72$\pm$2.6&90.78$\pm$0.3&91.97$\pm$0.2&\underline{93.06$\pm$0.4}&
    \yjx{\textbf{93.89$\pm$0.3}}\\
    &&40&73.61$\pm$8.6&90.00$\pm$0.3&89.14$\pm$0.2&90.15$\pm$0.4&89.03$\pm$0.7&89.47$\pm$0.9&89.22$\pm$0.5&89.92$\pm$0.4&90.76$\pm$0.3&\underline{91.77$\pm$0.2}&
    \yjx{\textbf{92.64$\pm$0.2}}\\
    &&60&74.05$\pm$8.3&90.95$\pm$0.2&91.17$\pm$0.1&91.01$\pm$0.3&90.43$\pm$0.6&90.34$\pm$0.8&90.35$\pm$0.8&90.66$\pm$0.5&91.59$\pm$0.2&\underline{94.13$\pm$0.2}&
    \yjx{\textbf{94.59$\pm$0.2}}\\
    \cline{2-14}
    &\multirow{3}{*}{AUC}
    &20&90.59$\pm$4.3&98.15$\pm$0.1&97.69$\pm$0.0&98.21$\pm$0.2&97.96$\pm$0.4&98.07$\pm$0.6&96.99$\pm$1.4&97.75$\pm$0.3&98.32$\pm$0.1&\underline{99.09$\pm$0.1}&
    \yjx{\textbf{99.11$\pm$0.1}}\\
    &&40&91.42$\pm$4.0&97.85$\pm$0.1&97.08$\pm$0.0&97.93$\pm$0.1&97.70$\pm$0.3&97.48$\pm$0.6&97.12$\pm$0.4&97.23$\pm$0.2&98.06$\pm$0.1&\underline{98.81$\pm$0.1}&
    \yjx{\textbf{99.09$\pm$0.0}}\\
    &&60&91.73$\pm$3.8&98.37$\pm$0.1&98.00$\pm$0.0&98.49$\pm$0.1&97.97$\pm$0.2&97.96$\pm$0.5&97.76$\pm$0.5&97.72$\pm$0.4&98.59$\pm$0.1&\underline{99.41$\pm$0.0}&
    \yjx{\textbf{99.41$\pm$0.1}}\\
    \hline
    \multirow{9}{*}{AMiner}&
    \multirow{3}{*}{Ma-F1}
    &20&42.46$\pm$2.5&60.22$\pm$2.0&54.78$\pm$0.5&58.32$\pm$1.1&50.06$\pm$0.9&56.07$\pm$3.2&51.61$\pm$3.2&59.50$\pm$2.1&\underline{71.38$\pm$1.1}&71.09$\pm$0.3&
    \yjx{\textbf{71.41$\pm$0.7}}\\
    &&40&45.77$\pm$1.5&65.66$\pm$1.5&64.77$\pm$0.5&64.50$\pm$0.7&58.97$\pm$0.9&63.85$\pm$1.5&54.72$\pm$2.6&61.92$\pm$2.1&\underline{73.75$\pm$0.5}&70.40$\pm$0.2&
    \yjx{\textbf{73.88$\pm$1.1}}\\
    &&60&44.91$\pm$2.0&63.74$\pm$1.6&60.65$\pm$0.3&65.53$\pm$0.7&57.34$\pm$1.4&62.02$\pm$1.2&55.45$\pm$2.4&61.15$\pm$2.5&\textbf{75.80$\pm$1.8}&72.82$\pm$0.5&
    \yjx{\underline{72.93$\pm$0.2}}\\
    \cline{2-14}
    &\multirow{3}{*}{Mi-F1}
    &20&49.68$\pm$3.1&65.78$\pm$2.9&60.82$\pm$0.4&63.64$\pm$1.1&61.49$\pm$2.5&68.86$\pm$4.6&62.39$\pm$3.9&63.93$\pm$3.3&\underline{78.81$\pm$1.3}&78.03$\pm$0.2&
    \yjx{\textbf{79.31$\pm$1.0}}\\
    &&40&52.10$\pm$2.2&71.34$\pm$1.8&69.66$\pm$0.6&71.57$\pm$0.7&68.47$\pm$2.2&76.89$\pm$1.6&63.87$\pm$2.9&63.60$\pm$2.5&\underline{80.53$\pm$0.7}&76.77$\pm$0.2&
    \yjx{\textbf{81.76$\pm$1.3}}\\
    &&60&51.36$\pm$2.2&67.70$\pm$1.9&63.92$\pm$0.5&69.76$\pm$0.8&65.61$\pm$2.2&74.73$\pm$1.4&63.10$\pm$3.0&62.51$\pm$2.6&\textbf{82.46$\pm$1.4}&78.88$\pm$0.3&
    \yjx{\underline{79.16$\pm$0.5}}\\
    \cline{2-14}
    &\multirow{3}{*}{AUC}
    &20&70.86$\pm$2.5&85.39$\pm$1.0&81.22$\pm$0.3&83.35$\pm$0.5&77.96$\pm$1.4&78.92$\pm$2.3&75.89$\pm$2.2&85.34$\pm$0.9&90.82$\pm$0.6&\textbf{92.89$\pm$0.1}&
    \yjx{\underline{91.02$\pm$1.0}}\\
    &&40&74.44$\pm$1.3&88.29$\pm$1.0&88.82$\pm$0.2&88.70$\pm$0.4&83.14$\pm$1.6&80.72$\pm$2.1&77.86$\pm$2.1&88.02$\pm$1.3&92.11$\pm$0.6&\textbf{92.88$\pm$0.1}&
    \yjx{\underline{92.72$\pm$1.0}}\\
    &&60&74.16$\pm$1.3&86.92$\pm$0.8&85.57$\pm$0.2&87.74$\pm$0.5&84.77$\pm$0.9&80.39$\pm$1.5&77.21$\pm$1.4&86.20$\pm$1.7&92.40$\pm$0.7&\underline{92.51$\pm$0.2}&
    \yjx{\textbf{92.72$\pm$0.5}}\\
    \hline

    \multirow{9}{*}{\yjx{IMDB}}&
    \multirow{3}{*}{\yjx{Ma-F1}}

    &\yjx{20}&
    \yjx{38.71$\pm$0.9} & 
    \yjx{46.72$\pm$1.2} & 
    \yjx{44.62$\pm$0.7} & 
    \yjx{43.28$\pm$1.1} & 
    \yjx{53.89$\pm$1.2} & 
    \yjx{45.97$\pm$0.4} & 
    \yjx{50.63$\pm$0.2} & 
    \yjx{55.95$\pm$1.4} & 
    \yjx{55.35$\pm$0.5} &
    \yjx{\underline{56.89$\pm$0.6}}  & \yjx{\textbf{62.91$\pm$0.6}} \\

    &&\yjx{40}& 
    \yjx{36.91$\pm$0.7} & 
    \yjx{52.21$\pm$0.7} & 
    \yjx{50.10$\pm$0.8} & 
    \yjx{50.21$\pm$1.5} & 
    \yjx{53.11$\pm$0.8} &
    \yjx{51.83$\pm$0.6} & 
    \yjx{52.36$\pm$0.3} & 
    \yjx{53.75$\pm$1.1} & 
    \yjx{57.09$\pm$0.8} & \yjx{\underline{58.19$\pm$0.4}} & \yjx{\textbf{58.96$\pm$0.6}} \\
    
    &&\yjx{60}&
    \yjx{37.28$\pm$1.0} & 
    \yjx{51.78$\pm$0.6} & 
    \yjx{53.94$\pm$0.5} & 
    \yjx{49.65$\pm$1.0} & 
    \yjx{55.57$\pm$0.7} &
    \yjx{52.43$\pm$0.6} &
    \yjx{53.09$\pm$0.4} & 
    \yjx{58.00$\pm$0.9} & 
    \yjx{56.74$\pm$0.7} & \yjx{\underline{60.03$\pm$0.4}} & \yjx{\textbf{62.04$\pm$0.3}} \\
    
    \cline{2-14}
    &\multirow{3}{*}{\yjx{Mi-F1}}

    &\yjx{20}& 
    \yjx{42.43$\pm$1.0} &
    \yjx{47.65$\pm$1.0} & 
    \yjx{45.10$\pm$0.8} & 
    \yjx{45.40$\pm$1.1} & 
    \yjx{55.18$\pm$1.1} &
    \yjx{46.50$\pm$0.6} & 
    \yjx{51.01$\pm$0.2} & 
    \yjx{56.32$\pm$1.4} & 
    \yjx{55.53$\pm$0.6} & \yjx{\underline{57.01$\pm$0.6}} & \yjx{\textbf{63.13$\pm$0.6}} \\
    
    &&\yjx{40}& 
    \yjx{40.63$\pm$1.4} & 
    \yjx{51.74$\pm$0.9} & 
    \yjx{50.13$\pm$0.8} & 
    \yjx{50.57$\pm$1.3} &
    \yjx{53.60$\pm$0.8} &
    \yjx{52.20$\pm$0.6} &
    \yjx{52.56$\pm$0.2} & 
    \yjx{54.13$\pm$1.0} & 
    \yjx{57.45$\pm$0.7} & 
    \yjx{\underline{58.76$\pm$0.4}} & 
    \yjx{\textbf{59.22$\pm$0.6}} \\
    
    &&\yjx{60}&
    \yjx{39.68$\pm$0.6} &  
    \yjx{51.26$\pm$0.8} & 
    \yjx{53.76$\pm$0.6} & 
    \yjx{49.44$\pm$1.0} &
    \yjx{55.47$\pm$0.8} &
    \yjx{52.09$\pm$0.6} & 
    \yjx{52.98$\pm$0.3} & 
    \yjx{57.78$\pm$0.9} & 
    \yjx{56.41$\pm$0.8} & \yjx{\underline{60.32$\pm$0.3}} & \yjx{\textbf{62.09$\pm$0.3}} \\
    
    \cline{2-14}
    &\multirow{3}{*}{\yjx{AUC}}   

    &\yjx{20}&
    \yjx{57.50$\pm$0.9} & 
    \yjx{65.96$\pm$1.0} & 
    \yjx{62.68$\pm$0.5} & 
    \yjx{61.94$\pm$1.1} &
    \yjx{73.15$\pm$0.7} &
    \yjx{64.94$\pm$0.3} & 
    \yjx{68.99$\pm$0.1} & 
    \yjx{74.52$\pm$0.9} & 
    \yjx{73.67$\pm$0.4} & \yjx{\underline{78.36$\pm$0.4}} & \yjx{\textbf{80.12$\pm$0.2}} \\
    
    &&\yjx{40}& 
    \yjx{55.89$\pm$0.7} & 
    \yjx{70.52$\pm$0.8} & 
    \yjx{68.73$\pm$0.5} & 
    \yjx{68.61$\pm$0.9} &
    \yjx{72.74$\pm$0.8} &
    \yjx{68.14$\pm$0.2} &
    \yjx{70.61$\pm$0.1} & 
    \yjx{73.23$\pm$0.8} & 
    \yjx{75.24$\pm$0.5} &
    \yjx{\underline{76.60$\pm$0.3}} &
    \yjx{\textbf{78.29$\pm$0.3}} \\
    
    &&\yjx{60}& 
    \yjx{57.31$\pm$0.4} & 
    \yjx{70.00$\pm$0.2} & 
    \yjx{72.09$\pm$0.3} &
    \yjx{67.95$\pm$0.6} &
    \yjx{73.65$\pm$0.3} &
    \yjx{70.00$\pm$0.2} & 
    \yjx{71.41$\pm$0.1} & 
    \yjx{76.11$\pm$0.6} & 
    \yjx{73.85$\pm$0.4} & \yjx{\underline{77.50$\pm$0.1}} & \yjx{\textbf{79.44$\pm$0.2}} \\
    \hline
  \end{tabular}}
  
\end{table*}

\noindent{\small$\bullet$} \textbf{IMDB:}
As a subset of Internet Movie Database,
the dataset contains 4275 moives (M), 5432 actors (A), 2083 directios (D) and 7313 keywords (K).
Links include M-A (an actor stars in a movie), M-D (a director directs a movie) and M-K (a movie contains a keyword).
We consider the meta-path set $\{$MAM, MDM, MKM$\}$.
Our task is to classify movies into three classes,
i.e., Action, Comedy and Drama.




\subsection{Baselines}
We compare \name~with 9 other state-of-the-art methods, which can be grouped into three categories:

\noindent{\small$\bullet$}\textbf{[Methods specially designed for homogeneous graphs]}: 
{GraphSAGE}~\cite{GraphSAGE} aggregates information from a fixed number of neighbors to generate nodes' embedding.
{GAE}~\cite{gae} is a generative method that generates representations by reconstructing the adjacency matrix.
{DGI}~\cite{dgi} maximizes the agreement between node representations and a global summary vector.

\noindent{\small$\bullet$}\textbf{[Semi-supervised learning methods in HINs]}: 
{HAN}~\cite{han} is proposed to learn node representations using node-level and semantic-level attention mechanisms.

\noindent{\small$\bullet$}\textbf{[Unsupervised learning methods in HINs]}: 
{HERec}~\cite{herec} utilizes the skip-gram model on each meta-path to embed induced graphs.
{HetGNN}~\cite{hetgnn} aggregates information from different types of neighbors based on random walk with start.
{DMGI}~\cite{dmgi} constructs contrastive learning between the original network and a corrupted network on each meta-path and adds a consensus regularization to fuse node embeddings from different meta-paths.
{Mp2vec}~\cite{metapath2vec} generates nodes' embedding vectors by performing meta-path-based random walks.
{HeCo}~\cite{heco} constructs two views with meta-paths and network schema to perform contrastive learning across them.
In particular,
HeCo is the state-of-the-art heterogeneous contrastive learning model.

\begin{table}
  \caption{Quantitative results (\%) on node clustering.}
  \label{clustering}
  \resizebox{\columnwidth}{!}{
  \begin{tabular}{c|cc|cc|cc|cc}
    \hline
    Datasets & \multicolumn{2}{c|}{ACM} & \multicolumn{2}{c|}{DBLP} & \multicolumn{2}{c|}{AMiner} & \multicolumn{2}{c}{\yjx{IMDB}}\\
    \hline
    Metrics & NMI & ARI & NMI & ARI & NMI & ARI & \yjx{NMI} & \yjx{ARI}\\
    \hline
GraphSage&29.20&27.72&51.50&36.40&15.74&10.10&\yjx{1.27} &\yjx{1.42}\\
GAE&27.42&24.49&72.59&77.31&28.58&20.90& \yjx{7.79} & \yjx{5.72}\\
Mp2vec&48.43&34.65&73.55&77.70&30.80&25.26&\yjx{9.89}&\yjx{10.92} \\
HERec&47.54&35.67&70.21&73.99&27.82&20.16& \yjx{0.49} & \yjx{0.43}\\
HetGNN&41.53&34.81&69.79&75.34&21.46&26.60&\yjx{11.74} & \yjx{13.24}\\
DGI&51.73&41.16&59.23&61.85&22.06&15.93&\yjx{7.43}&\yjx{5.15}\\
DMGI&51.66&46.64&70.06&75.46&19.24&20.09&\yjx{10.17}&\yjx{9.43}\\
HeCo&56.87&56.94&74.51&80.17&32.26&28.64& \yjx{11.16}&\yjx{12.68}\\

\hline
\name&\underline{66.21}&\underline{71.17}&\underline{75.46}&\underline{81.19}&\underline{33.91}&\textbf{33.81}&\yjx{\underline{14.88}} & \yjx{\underline{15.77}}\\

\nameplus & \textbf{67.74} & \textbf{74.11} & \textbf{77.33} & \textbf{82.30} & \textbf{40.56} & \underline{32.34} &  \yjx{\textbf{15.44}} & \yjx{\textbf{16.29}} \\

    \hline
  \end{tabular}}
\end{table}

\subsection{Experimental Setup}
We implement \name~with PyTorch and adopt the Adam optimizer to train the model.
We fine-tune the learning rate from $\{$5e-4, 6e-4, 7e-4$\}$,
the penalty weight on the ${l}_2$-norm regularizer from \{0, 1e-4, 1e-3\} and the patience for early stopping from 10 to 40 with step size 5, 
i.e., we stop training if the total loss does not decrease for patience consecutive epochs.
We set the dropout rate ranging from 0.0 to 0.9, and the temperature {$\tau$} in Eq.~\ref{eq:meow_con} from 0.1 to 1.0, both with step size 0.1.
We set $K$ in the neighbor filtering 
based on the average number of connections of all the objects under each meta-path.
For data augmentation,
we fine-tune the masking rate for both features and edges from 0.0 to 0.6 with step size 0.1.
We perform node clustering twice and set $\alpha=5$ in all datasets.
Further, 
we set the number of clusters $U$ to $\{$100, 300$\}$, $\{$200, 700$\}$, $\{$500, 1200$\}$, and $\{$100, 500$\}$ in ACM, DBLP, Aminer and IMDB, respectively.
We fine-tune the regularization weights $\lambda$ in prototypical contrastive learning from $\{$0.1, 1, 10$\}$.
For Aminer,
since nodes are not associated with features,
we first run metapath2vec with the default parameter settings from the original codes provided by the authors to construct nodes' initial feature vectors.
For fair comparison, we set the embedding dimension as 64 and randomly run the experiments 10 times, and report the average results for all the methods.
{For other competitors,
their results are directly reported from~\cite{heco}.}
We run all the experiments on a server
with 32G memory and a single Tesla V100 GPU.
We provide our code and data here:
\url{https://github.com/jianxiangyu/MEOW}.

\begin{figure*}[t]
    \centering
    \includegraphics[width=1.0\textwidth]{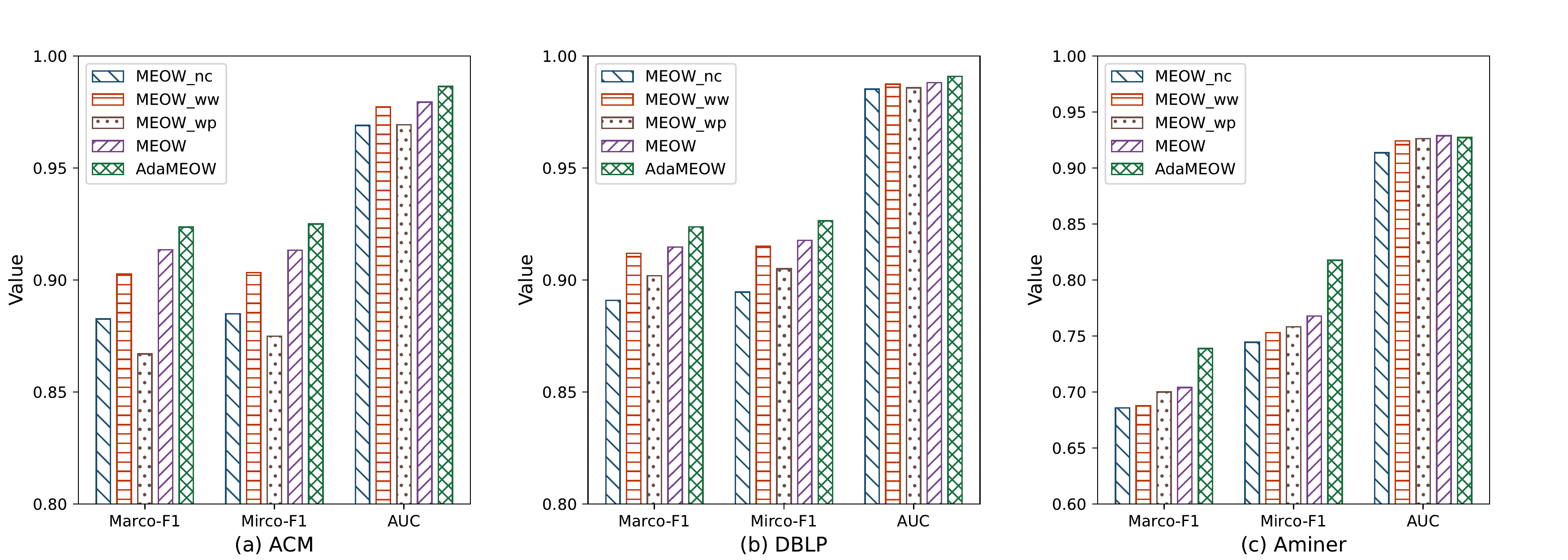}
    \caption{The ablation study results of 40 labeled nodes per class.}
    \label{fig:ablation_all}
\end{figure*}



\subsection{Node Classification}
We use the learned node embeddings to train a linear classifier to evaluate our model.
We randomly choose 20, 40, 60 labeled nodes per class as training set, and 1000 nodes as validation set and 1000 nodes for testing.
We use Macro-F1, Micro-F1 and AUC
as evaluation metrics.
For all the metrics,
the larger the value,
the better the model performance.
The results are reported in Table \ref{classification}. 
From the table,
we see that
\name~achieves the suboptimal performance on ACM, DBLP and IMDB, and 
performs very well on Aminer
in all the data splits.
This shows the importance of meta-path contextual information and the validity of the contrastive views we designed.
Compared with the state-of-the-art graph contrastive learning model Heco,
\name~achieves better performance on ACM, DBLP and IMDB.
For example, 
the Macro-F1 score and the Micro-F1 score of Heco is 90.64\% and 91.59\% with 60 labeled nodes per class on DBLP,
while \name~is 93.49\% and 94.13\%.
These results show the effectiveness of \name.
While \name\ performs slightly worse than Heco in Macro-F1 and Micro-F1 on Aminer, 
it outperforms Heco in the AUC scores.
This can be explained by the label imbalance on Aminer.
Specifically,
the number of objects in the label which has 
the maximum number of nodes
is $\sim 7$ times more than
that in the label which has the minimum number of nodes.
It is well known that
when labeled objects are imbalanced,
AUC is a more accurate metric than the other two.
This further verifies that \name\ is effective.

\yjx{
\nameplus\ achieves the best performance in most 36 cases.
On the basis of \name,
\nameplus\ has been improved on each dataset,
especially IMDB.
In the IMDB dataset, with 20 labeled nodes per class, the Micro-F1 score of \nameplus\ is 62.91\% and the Macro-F1 score is 63.13\% while the runners-up scores are only 56.89\% and 57.01\%. 
This can demonstrate that adaptive weights have stronger learning capability on datasets with more noise.
}

\subsection{Node Clustering}
We further perform K-means clustering to verify the quality of learned node embeddings.
We adopt normalized mutual information (NMI) and adjusted rand index (ARI) as the evaluation metrics.
For both metrics,
the larger, the better.
The results are reported in Table \ref{clustering}.
As we can see,
on the ACM dataset,
\name~obtains about 16\% improvements on NMI and 25\% improvements on ARI  compared to the best of the benchmark methods, 
demonstrating the superiority of our model.
This is because 
the prototypical contrastive learning drives node representations to be more compact in the same cluster,
which helps boost node clustering.
\nameplus\ can further make the boundaries between classes more distinct, resulting in better performance than \name.

\begin{figure*}[t]
    \centering
    \includegraphics[width=0.9\textwidth]{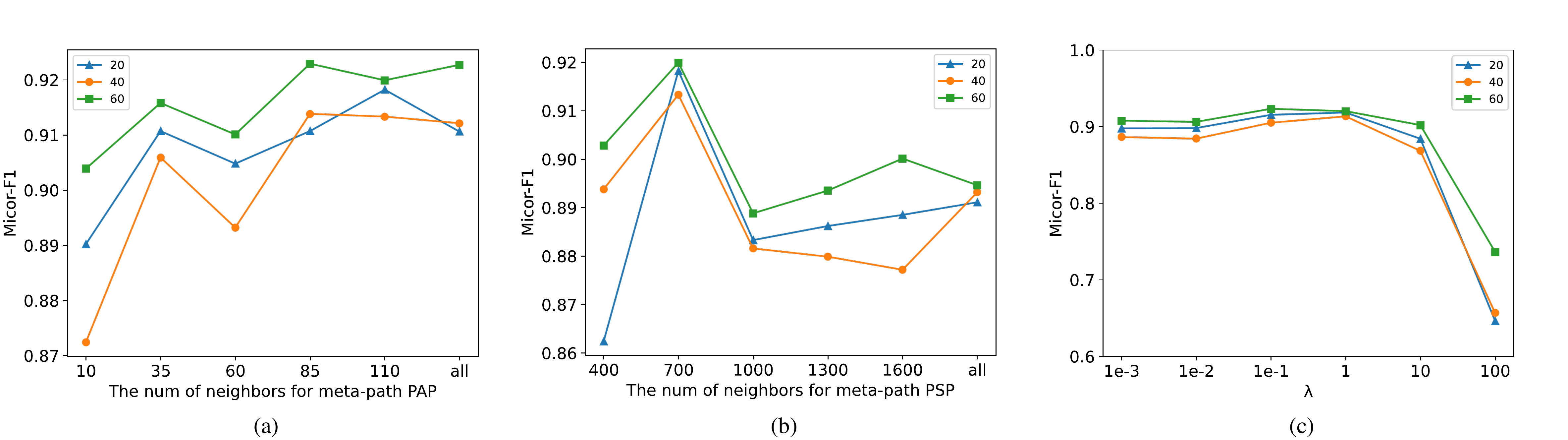}
    \caption{Hyper-parameter analysis on the ACM dataset. Here, $K$ is  the number of selected relevant neighbors w.r.t. a meta-path and $\lambda$ controls
the relative importance of two components of the loss function in Eq.~\ref{meow_loss}.}
    \label{fig:hyper}
\end{figure*}

\begin{figure*}[!t]
\centering
\subfloat[]{\includegraphics[width=2.0in]{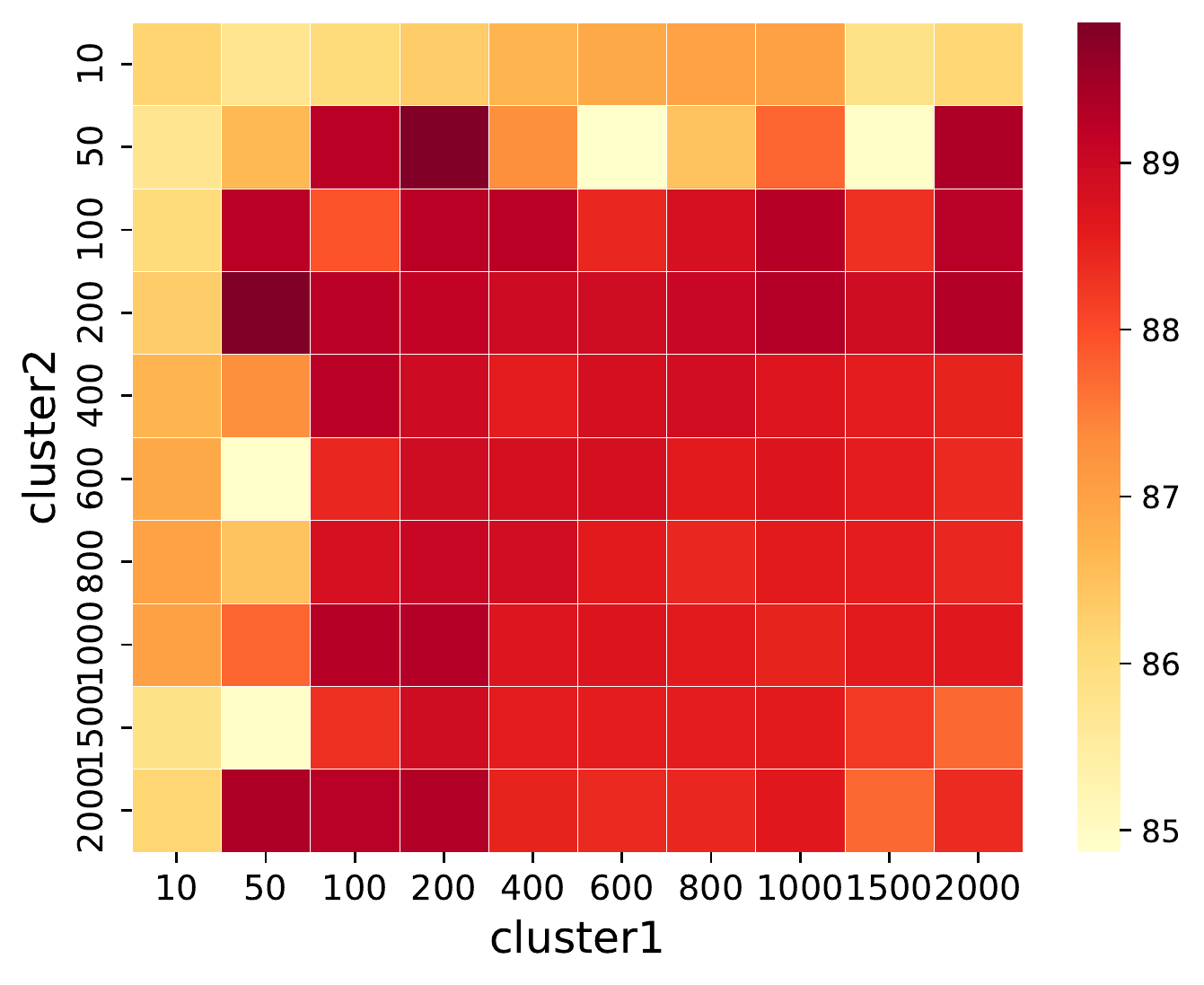}%
}
\hfil
\subfloat[]{\includegraphics[width=2.0in]{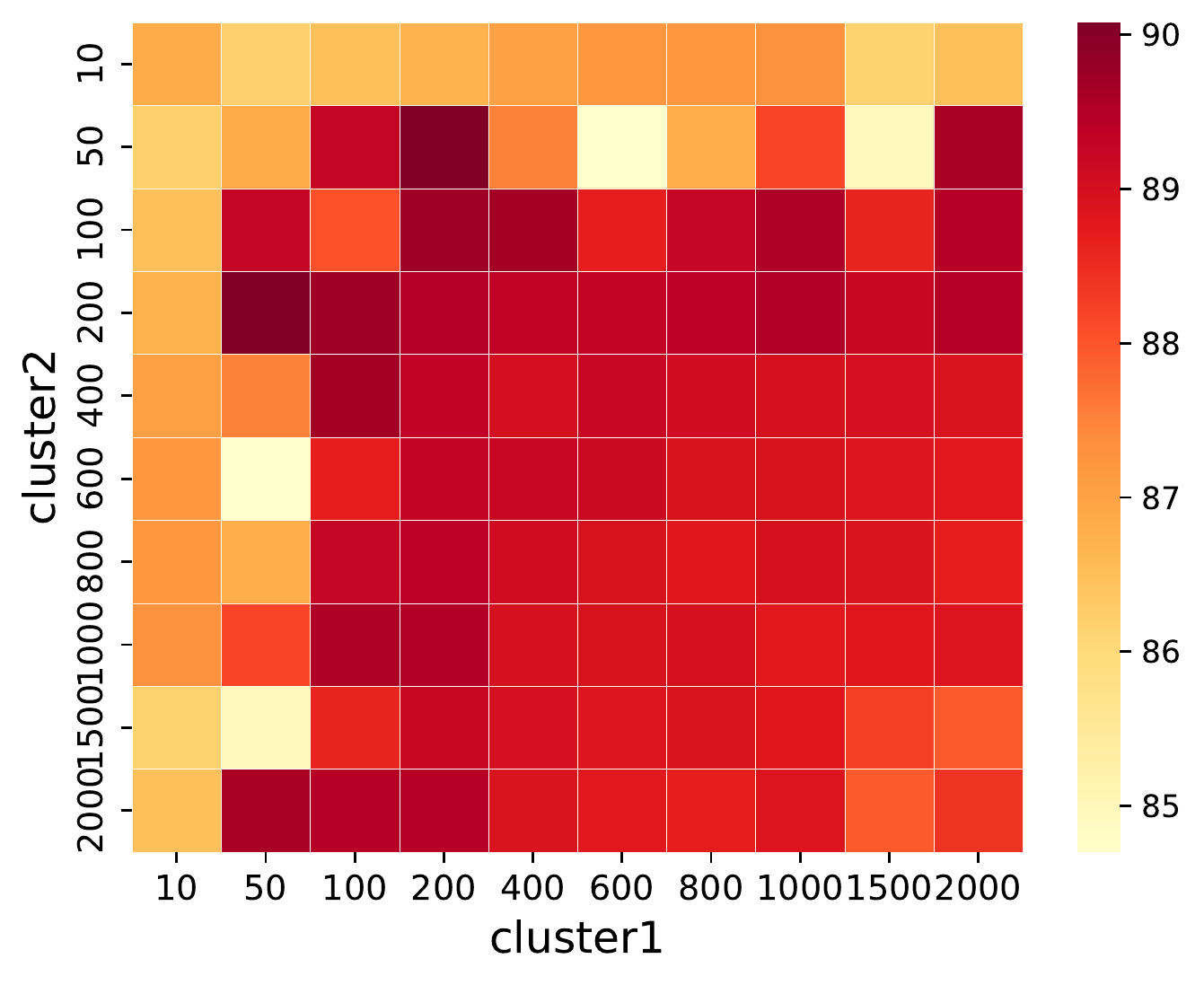}%
}
\hfil
\subfloat[]{\includegraphics[width=2.0in]{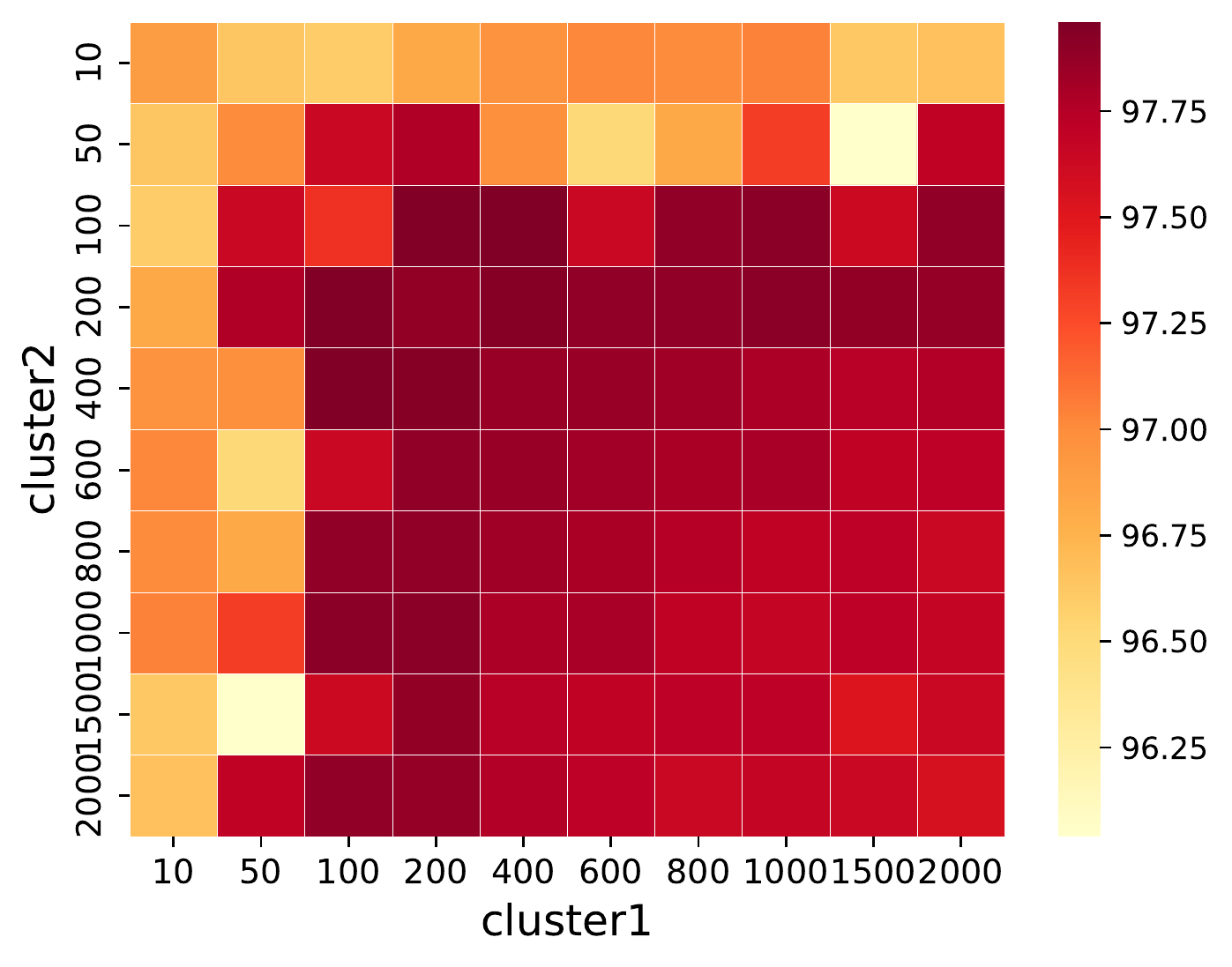}%
}
\caption{Hyper-parameter analysis (the number of clusters) of 40 labeled nodes per class on the ACM dataset. (a) Ma-F1 I. (b) Mi-F1. (c) AUC.}
\label{fig:hyper_cluster}
\end{figure*}

\subsection{Ablation Study}
We conduct an ablation study on \name~and \nameplus~to understand the characteristics of its main components.
To show the importance of the
prototypical contrastive learning regularization, 
we train the model with
$\mathcal{L}^{con}$ only and call this variant \textbf{\name\_wp} (\textbf{w}ithout \textbf{p}rototypical contrastive learning).
To demonstrate the importance of 
distinguishing negative samples with different characteristics, 
another variant is to not learn the weights of negative samples.
We call this \textbf{MEOW\_ww} (\textbf{w}ithout \textbf{w}eight).
Moreover,
we update nodes’ embeddings by aggregating information without considering meta-path contexts in the fine-grained view
and call this variant  \textbf{MEOW\_nc} (\textbf{n}o \textbf{c}ontext).
This helps us understand the importance of including meta-path contexts in heterogeneous graph contrastive learning.
We report the results of 40 labeled nodes per class, which is shown in Fig.\ref{fig:ablation_all}.
From these figures,
\name~achieves better performance than \name\_wp.
This is because
the prototypical contrastive learning can drive nodes of the same label to be more compact in the latent space, 
which leads to better classification results.
\name~outperforms \name\_ww on three datasets.
This further demonstrates the advantage of weighted negative samples.
In addition, \name~beats \name\_nc in all cases.
This shows that when using meta-paths, the inclusion of meta-path contexts is essential for effective heterogeneous graph contrastive learning.
\yjx{
The \nameplus~model 
outperforms others in most cases.
This
demonstrates the significance of
capturing the characteristics of 
negative samples and the effectiveness of learning their weights adaptively.
}

\subsection{Hyper-parameter Analysis}
We further perform a sensitivity analysis on the hyper-parameters of our method. 
In particular, we study three main hyper-parameters in \name: the number of selected
relevant neighbor in \emph{Pathsim}
,
the relative importance $\lambda$ of the two components of the loss function in Eq. \ref{meow_loss} on the ACM dataset
and
the number of clusters mentioned in Eq.~\ref{eq:meow_con} and Eq.~\ref{eq:meow_pro_con}.
In our experiments, 
we vary one parameter each time with others fixed.
Figure~\ref{fig:hyper} 
illustrates the results of the first two hyperparameters with 20, 40, 60
labeled nodes per class w.r.t. the Micro-F1 scores. 
(Results on Macro-F1 and AUC scores exhibit similar trends, and thus are omitted for space
limitation.) 
Figure~\ref{fig:hyper_cluster} displays the Macro-f1, Micro-f1, and AUC scores with 40 labeled nodes per class for the number of clusters.
The diagonal represents the results for clustering once, while the off-diagonal entries represent the results for clustering twice.
From the figure,
we see that

1) \,
In the case of meta-path PAP \emph{(Paper-Author-Paper)}, 
the more neighbors selected, 
the better the performance of the model. 
However, for meta-path PSP \emph{(Paper-Subject-Paper)},
we find that the Micro-F1 score first rises and then drops, as the number of neighbors increases. 
This is because the co-authored papers are more likely to be in the same area, while 
papers in the same subject could be from different research domains.
With the increase of neighbors,
more noisy connections induced by PSP could degrade the model performance.

2) \,
For the weight $\lambda$ that controls the importance of the prototypical contrastive loss function,
\name~
gives very stable performances over a wide range of parameter values.
The Micro-F1 score largely decreases when $\lambda$ is large enough.
This is because
a larger $\lambda$ will encourage more compactness within each class. 
However, 
this may cause some hard samples to be assigned to the incorrect clusters 
and cannot be corrected during the training process, resulting in misclassification. 

3) \,
From the heat map,
we can observe that when the clustering size is small, 
samples from different classes may mix within the same cluster,
which is detrimental to both contrastive loss and prototypical loss, 
resulting in poor model performance.
When the clustering size is too large, 
each node forms a cluster with its most similar node,
or even each node forms an individual cluster. 
In this case, 
the prototypical loss and contrastive loss become similar,
resulting in the weighted contrastive loss not having a significant impact and causing a slight decrease in performance.

\subsection{Case study}

\begin{table}
\centering
  \caption{Case study on different types of weights on the ACM dataset.}
  \label{weight}
  \resizebox{0.95\columnwidth}{!}{
  \begin{tabular}{c|ccc|ccc}
    \hline
    Metric & \multicolumn{3}{c|}{Ma-F1} & \multicolumn{3}{c}{Mif1}  \\
    \hline
    Split & 20 & 40 & 60 & 20 & 40 & 60 \\

    \hline
NW & 91.00 & 90.68 & 91.89 & 91.31 &  90.93 & 91.88  \\
RW & 91.43 & 90.64 & 92.49 & 91.54 & 90.94 & 92.43 \\
AW & 93.56 & 92.49 & 93.52 & 93.48 & 92.49 & 93.37  \\

    \hline
  \end{tabular}}
\end{table}

\begin{figure}[t]
\centering
\includegraphics[width=0.4\textwidth]{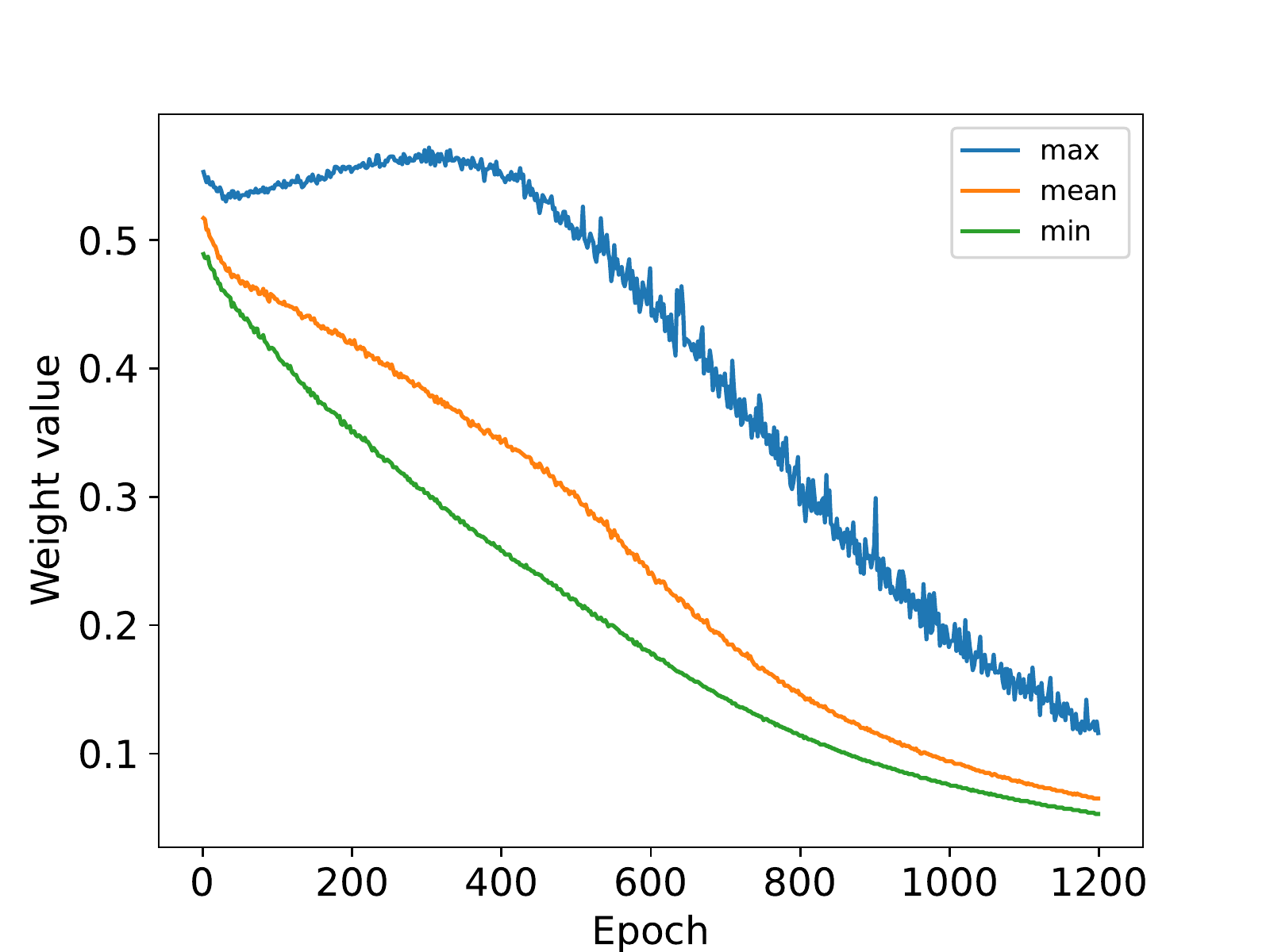}
    \caption{Dynamic changes in weights during training on the ACM dataset.}
    \label{fig:weight_var}
\end{figure}

In this section, 
we analyze the learned weights in Equation~\ref{eq:ada} through experiments.
First,
we evaluate the advantages of weighted InfoNCE using three cases:
\textbf{NW}(\textbf{n}o \textbf{w}eight), \textbf{RW}(\textbf{r}andom \textbf{w}eights), and \textbf{AW}(\textbf{a}daptive \textbf{w}eights). 
NW means all weights are set to 1, 
which is equivalent to regular InfoNCE. 
RW means we randomly assign a weight between 0 and 1 to each node pair in each epoch. 
AW refers to our proposed variant model~\nameplus.
As can be seen in Table~\ref{weight}, 
we can observe that compared to NW, 
using random weights leads to some improvement in results.
This is because different weight assignments of node pairs can influence the optimization direction of the model.
However, 
compared to AW, 
RW lacks stability during training and does not consider the characteristics of the node pairs.
Therefore, the results obtained with adaptive weights outperform the other two cases, demonstrating better performance.
Second,
in Figure~\ref{fig:weight_var}
we take the maximum, mean, and minimum values of the weights for all node pairs to reflect the dynamic changes of the weights during the training process.
We can observe that while there are some fluctuations in the weight changes, overall they exhibit a stable decreasing trend.
This is because as the training progresses, the nodes in the latent space gradually acquire more discriminative representations, requiring only small gradient values for fine-tuning.

Finally,
after training on the ACM dataset for 500 epochs, 
we randomly select an anchor and show the learned weights for its negative samples in Figure~\ref{sim_weight}. 
We can see that the overall trend of the weights is consistent with our expectations, 
as $\widetilde{\gamma_{ij}}$ can adaptively adjust its magnitude based on the characteristics of the samples. 
For false negative samples with high similarity, $\widetilde{\gamma_{ij}}$ is relatively small to ensure that they are not pushed away from the anchor.
For hard negative samples, $\widetilde{\gamma_{ij}}$ is expected to be larger, so that the anchor and hard negative samples can be distinguished. 
Further,
based on Equation~\ref{eq:theorem2},
for easy negative samples that have small similarity with the anchor, 
they can be pushed farther only when 
the weight $\widetilde{\gamma_{ij}}$ is set large.
All these results explain the reason why our experiment works better than other baselines.

\begin{figure}
    \centering
    \includegraphics[width=0.4\textwidth]{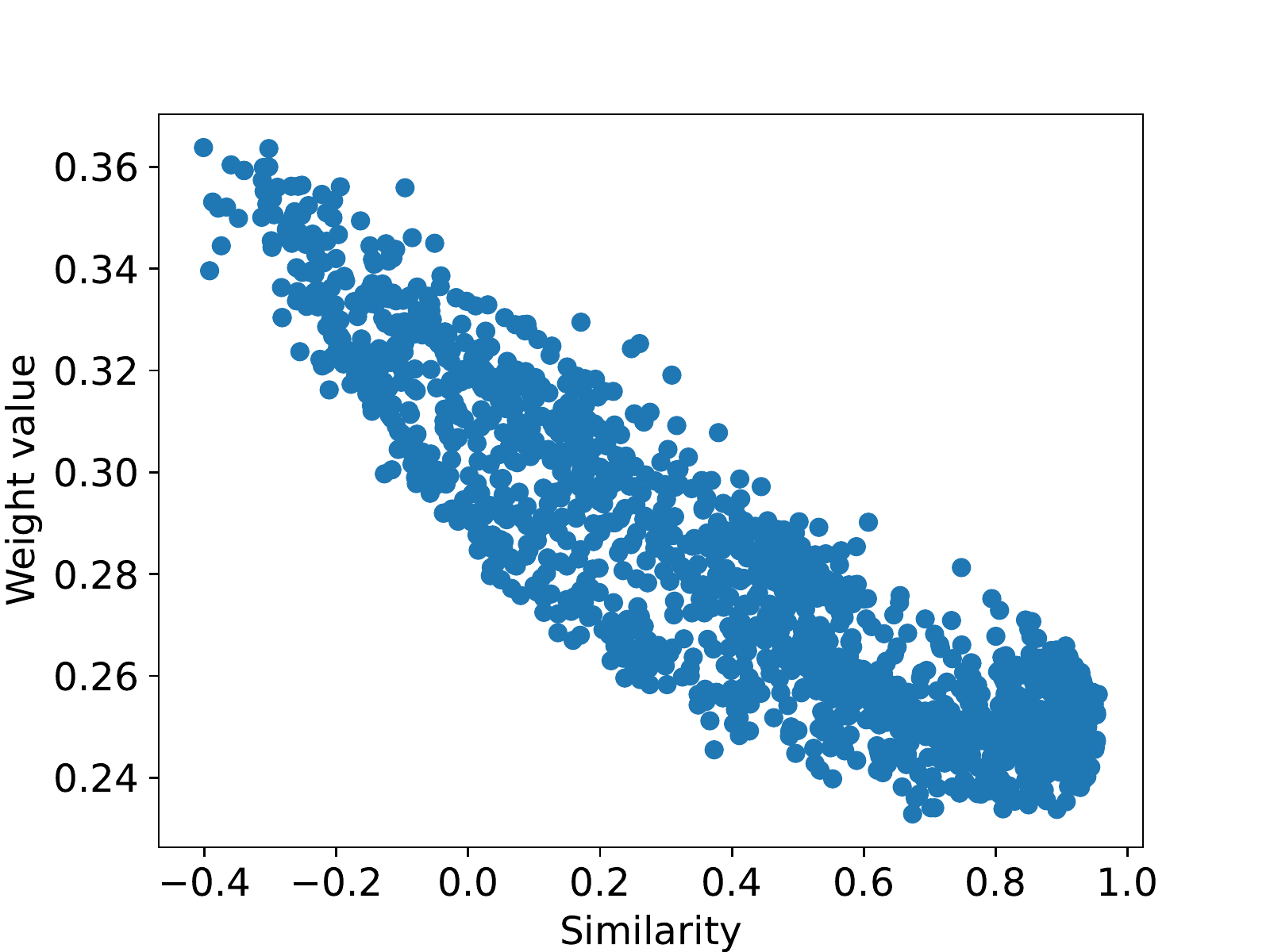}
    \caption{
    The relationship between similarity values and learned weights $\widetilde{\gamma_{ij}}$ of negative samples in Equation~\ref{eq:ada} for a randomly selected anchor after training 500 epochs on the ACM dataset.
    }
    \label{sim_weight}
  \end{figure}

\section{Conclusion}

We studied graph contrastive learning in HINs and proposed the \name\ model, 
which considers both meta-path contexts and weighted negative samples.
Specifically, \name\ constructs a coarse view and a fine-grained view for contrast.
In the coarse view, 
we took node embeddings derived by directly aggregating all the meta-paths as anchors, 
while in the fine-grained view, 
we utilized meta-path contexts
and constructed positive and negative samples for anchors.
Afterwards, 
we conducted a theoretical analysis of the InfoNCE loss and recognized its limitations for negative sample gradient magnitudes. Therefore, we proposed a weighted loss function for negative samples.
In \name~,
we distinguished hard negatives from false ones by performing node clustering and 
using the results to 
assign weights to negative samples.
Additionally, 
we introduced prototypical contrastive learning, which helps learn compact embeddings of nodes in the same cluster.
Further, 
we proposed a variant model called~\nameplus~which can adaptively learn soft-valued weights for negative samples instead of hard-valued weights in~\name. 
Finally, we conducted extensive experiments to show the superiority of~\name~and~\nameplus~against other SOTA methods.

\section*{Acknowledgments}
This work is supported by National Natural Science Foundation of China No. 62202172 and Shanghai Science and Technology Committee General Program No. 22ZR1419900.


\bibliographystyle{IEEEtran}
\bibliography{IEEEabrv,myrefs}


\vspace{11pt}

\begin{IEEEbiography}
[{\includegraphics[width=1in,height=1.25in,clip,keepaspectratio]{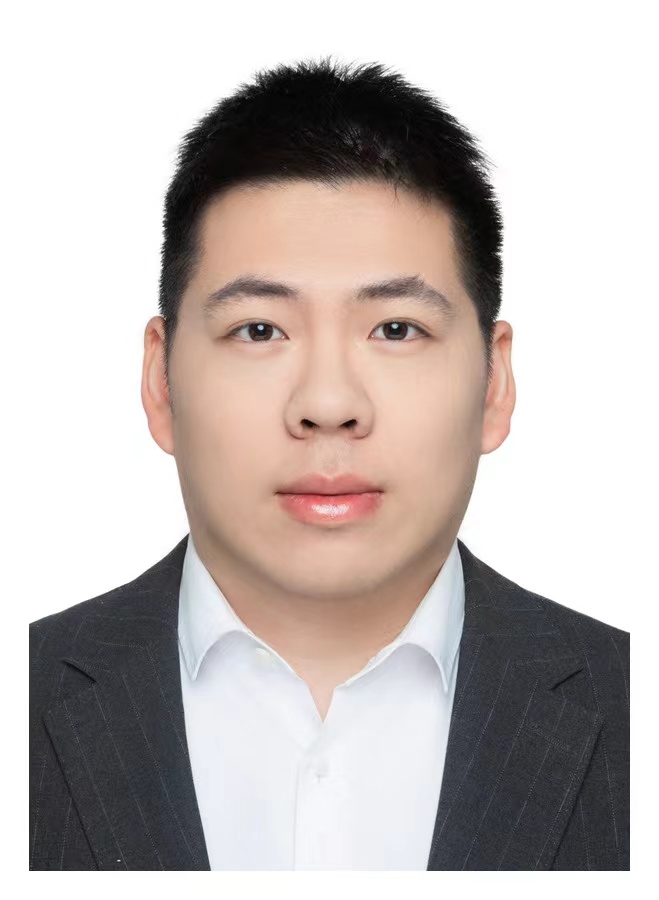}}]{Jianxiang Yu}
is currently working toward the MS degree at the School of Data Science and Engineering in East China Normal University.
His research interests include graph neural networks and data mining.
\end{IEEEbiography}

\begin{IEEEbiography}
[{\includegraphics[width=1in,height=1.25in,clip,keepaspectratio]{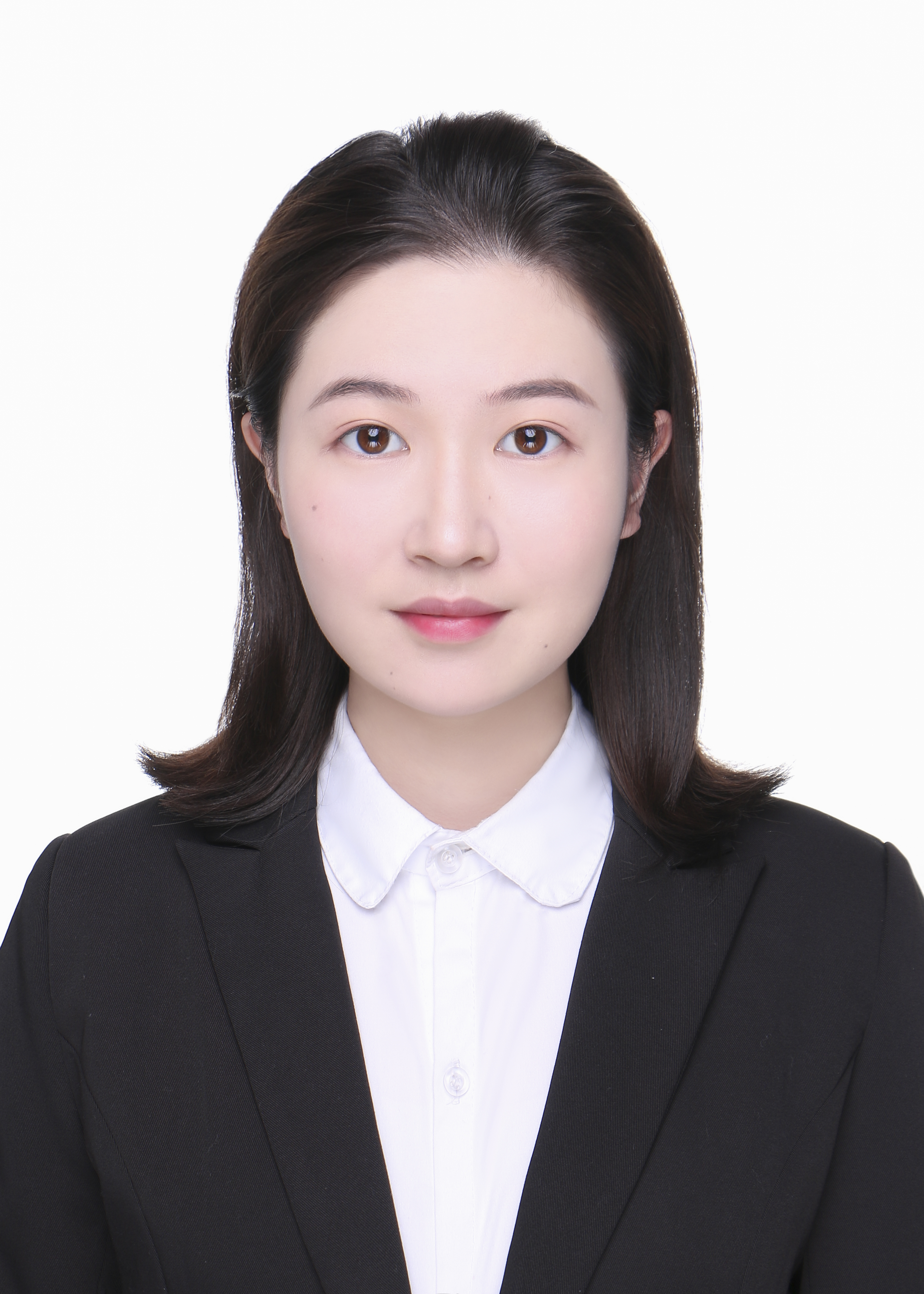}}]
{Qingqing Ge}
is currently working toward the MS degree at the School of Data Science and Engineering in East China Normal University.
Her general research interests include weakly supervised graph neural networks and prompt learning.
\end{IEEEbiography}

\begin{IEEEbiography}[{\includegraphics[width=1in,height=1.25in,clip,keepaspectratio]{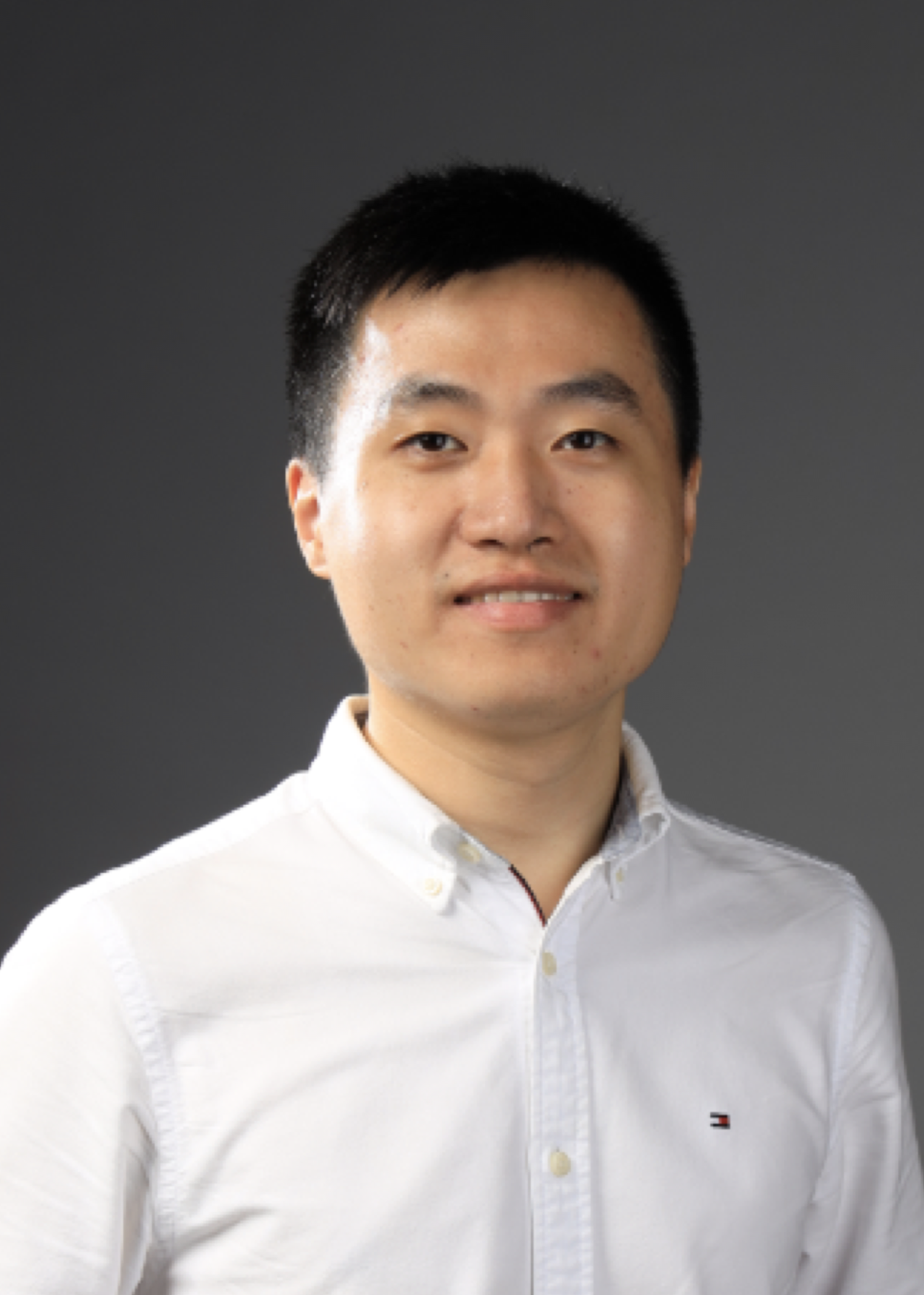}}]{Xiang~Li} received his Ph.D. degree from the University of Hong Kong in 2018. From 2018 to 2020, 
he worked as a research scientist in the Data Science Lab at JD.com and a research associate at The University of Hong Kong, respectively. 
He is currently a research professor at the School of Data Science and Engineering in East China Normal University. 
His general research interests include data mining and machine learning applications.
\end{IEEEbiography}

\begin{IEEEbiography}[{\includegraphics[width=1in,height=1.25in,clip,keepaspectratio]{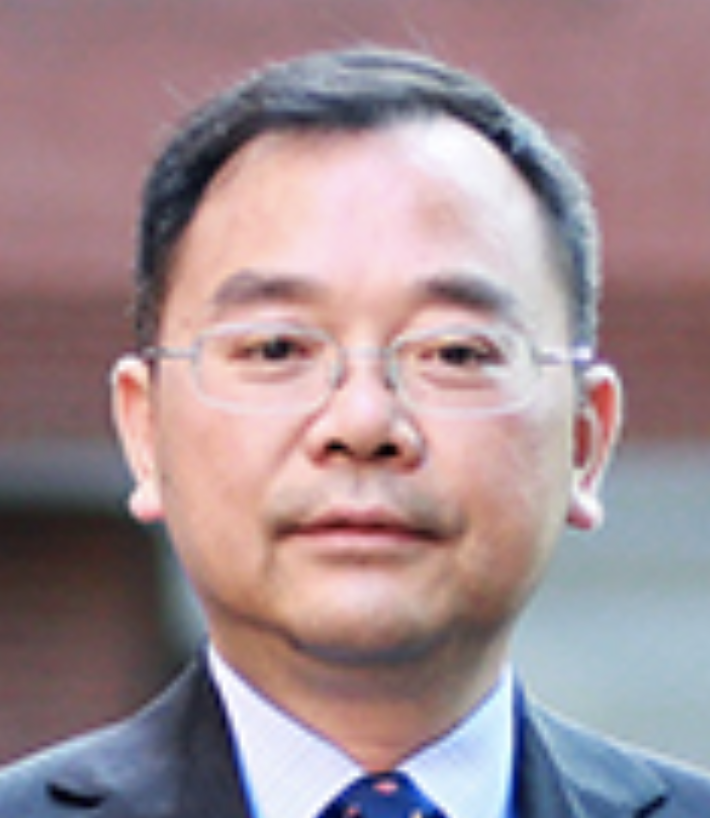}}]{Aoying~Zhou} is a professor at the School of Data Science and Engineering with East China Normal University (ECNU), where he is also the Vice Chancellor of the University.
Before joining ECNU in 2008, he worked for Fudan University at the Computer Science Department for 15 years. He is the winner of the National Science Fund for Distinguished Young Scholars supported by NSFC and the professorship appointment under Changjiang Scholars Program of Ministry of Education. 
He acted as a vice-director of ACM SIGMOD China and Database Technology Committee of China Computer Federation. 
He is serving as a member of the editorial boards, including the VLDB Journal, the WWW Journal, and etc. 
His research interests include data management, in-memory cluster computing, big data benchmarking, and performance optimization.
\end{IEEEbiography}

\vfill

\end{document}